\documentclass[12pt]{article}
\usepackage{amsmath}
\usepackage{graphicx}
\usepackage{enumerate}
\usepackage{natbib}
\usepackage{url} 

\newcommand{\blind}{0}

\usepackage{mathtools}
\usepackage{amsfonts}
\usepackage{algorithm}      
\usepackage[noend]{algpseudocode} 
\usepackage{wrapfig}
\usepackage{caption}
\usepackage{subcaption}
\usepackage{xcolor}         
\RequirePackage[colorlinks,citecolor=blue,urlcolor=blue]{hyperref}
\usepackage{amssymb}

\newcommand{\BlackBox}{\rule{1.5ex}{1.5ex}}  

\newtheorem{theorem}{Theorem}
\newtheorem{definition}{Definition}

\newtheorem{remark}{Remark}
\newtheorem{corollary}{Corollary}

\def\blue{\color{blue}}

\def\Vsc{\mathcal{V}}
\def\trans{^{\scriptscriptstyle \sf T}}

\def\Dsc{\mathcal{D}}

\def \be{\begin{equs}}
\def \ee{\end{equs}}

\def\bH{\mathbf{H}}

\def\bU{\mathbf{U}}

\def\bX{\mathbf{X}}
\def\bW{\mathbf{W}}
\def\bI{\mathbf{I}}

\def\bA{\mathbf{A}}
\def\bB{\mathbf{B}}
\def\bP{\mathbf{P}}
\def\bD{\mathbf{D}}

\def\Asc{\mathcal{A}}

\def\Dsc{\mathcal{D}}
\def\Ssc{\mathcal{S}}
\def\Esc{\mathcal{E}}
\def\Fsc{\mathcal{F}}
\def\Csc{\mathcal{C}}
\def\bzero{\mathbf{0}}
\def\bone{\mathbf{1}}

\def\Pbb{\mathbb{P}}

\def\btheta{\boldsymbol{\theta}}
\def\bxi{\boldsymbol{\xi}}
\def\bgamma{\boldsymbol{\gamma}}
\def\bbeta{\boldsymbol{\beta}}

\def\bEta{\boldsymbol{\eta}}

\def\bTheta{\boldsymbol{\Theta}}

\def\bTheta{\boldsymbol{\Theta}}
\def\bmu{\boldsymbol{\mu}}

\def\P{\mathbb{P}}

\def\E{\mathbb{E}}
\def\R{\mathbb{R}}
\def\B{\mathbb{B}}
\def\N{\mathbb{N}}

\def\argmindum{\mathop{\mbox{argmin}}}

\def\argmaxdum{\mathop{\mbox{argmax}}}

\def\cC{\mathcal{C}}

\begin{document}

\def\spacingset#1{\renewcommand{\baselinestretch}%
{#1}\small\normalsize} \spacingset{1}


\if0\blind
{
  \title{\bf Federated Offline Reinforcement Learning}
  \author{ Doudou Zhou$^{1*}$,  
Yufeng Zhang$^{2}$\footnote{Zhou and Zhang contributed equally}, Aaron Sonabend-W\textit{}$^{1}$, \\
Zhaoran Wang$^{2}$, Junwei Lu$^{1}$, Tianxi Cai$^{1,3}$\bigskip \\
\small 
$^1${Department of Biostatistics, Harvard T.H. Chan School of Public Health} \\
\small 
$^2${Departments of Industrial Engineering and Management Sciences, Northwestern University} \\
\small 
$^3${Department of Biomedical Informatics, Harvard Medical School}
}
  \maketitle
} \fi

\if1\blind
{
  \bigskip
  \bigskip
  \bigskip
  \begin{center}
    {\LARGE\bf Federated Offline Reinforcement Learning}
\end{center}
  \medskip
} \fi

\bigskip
\begin{abstract}
Evidence-based or data-driven dynamic treatment regimes are essential for personalized medicine, which can benefit from offline reinforcement learning (RL). Although massive healthcare data are available across medical institutions, they are prohibited from sharing due to privacy constraints. Besides, heterogeneity exists in different sites. As a result, federated offline RL algorithms are necessary and promising to deal with the problems. In this paper, we propose a multi-site Markov decision process model that allows for both homogeneous and heterogeneous effects across sites. The proposed model makes the analysis of the site-level features possible. We design the first federated policy optimization algorithm for offline RL with sample complexity. The proposed algorithm is communication-efficient, which requires only a single round of communication interaction by exchanging summary statistics. We give a theoretical guarantee for the proposed algorithm, 
where the suboptimality for the learned policies is comparable to the rate as if data is not distributed. Extensive simulations demonstrate the effectiveness of the proposed algorithm. The method is applied to a sepsis dataset in multiple sites to illustrate its use in clinical settings. 
\end{abstract}

\noindent%
\vfill

\pagenumbering{arabic}
\newpage

\spacingset{1.9} 
\section{Introduction}
\label{sec: intro}
The construction of evidence-based or data-driven dynamic treatment regimes (DTRs) \citep{murphy2001marginal, lavori2004dynamic} is a central problem for personalized medicine. Reinforcement learning (RL) models, which often treat the observations by the episodic Markov decision process (MDP) \citep{ sutton2018reinforcement} have shown remarkable success for these applications. For instance, RL has been used for deciding optimal treatments and medication dosage in sepsis management \citep{raghu2017continuous, sonabend2020expert} and HIV therapy selection \citep{parbhoo2017combining}. 

However, many healthcare problems only allow for retrospective studies due to their offline nature \citep{lange2012batch}. Exploring new policies/treatments on patients is subject to ethical, financial, and legal constraints due to their associated significant risks and costs. Hence, it is typically infeasible or impractical to collect data online for finding DTR. With increasingly available massive healthcare datasets such as the electronic
health records (EHR) data \citep{johnson2016mimic} and mobile health data \citep{xu2021fedmood}, and due to the high risk of direct interventions on patients and the expensive cost of conducting clinical trials \citep{chakraborty2014dynamic}, offline RL is usually preferred for finding DTRs \citep{murphy2003optimal,robins2004optimal, chakraborty2013statistical},  whose objective is to learn an optimal policy based on existing datasets \citep{ levine2020offline,  kumar2019stabilizing,sonabend2020semi}. For example, mobile health data are employed for controlling blood glucose levels in patients with type $1$ diabetes \citep{luckett2019estimating} and the diagnosis of depression \citep{xu2021fedmood} due to the increasing popularity of mobile devices \citep{free2013effectiveness}.

However, healthcare data is often distributed across different sites. For instance, the EHR data are always stored locally in different hospitals, and the mobile health data are recorded and stored in the users' own mobile devices \citep{xu2021fedmood}. Although the aggregation of multi-site data can improve the quality of the models, for the protection of individual information \citep{agu2013smartphone, cao2017deepmood}, different medical institutions are often not allowed to share data \citep{hao2019efficient}, which hinders the direct aggregation of multi-site data \citep{duchi2014privacy} and affects the model accuracy significantly \citep{li2019privacy}. 

Besides, heterogeneity exists in different sites. For instance, patients are only given some specific drugs in a hospital due to the local suppliers, so models trained on a single site can mislead agents because of the limited actions. Similarly, doctors in the same hospital may follow a common treatment procedure, yielding insufficient exploration of the MDP. Thus, the trajectories in one hospital may distribute substantially differently from that induced by the optimal policy, which is known as distribution shift \citep{levine2020offline}. 

\subsection{Overview of the Proposed Model and Algorithm}

To resolve these challenges, in this paper, we propose a multi-site MDP model that allows heterogeneity among different sites to address previously discussed issues. Specifically, we consider the setting with $K$ sites generated from the episodic linear MDP \citep{puterman2014markov, sutton2018reinforcement} with the state space $\mathcal{S}$, action space $\Asc$ and horizon $H$ detailed in Section \ref{sec: problem}. We incorporate heterogeneity by allowing the transition kernel $\mathcal{P}^k=\big\{\P^k_{h}\big\}_{h=1}^{H}$ and the reward function $r^k =\big\{r_{h}^k\big\}_{h=1}^{H}$ to be different for the $K$ sites specified as follows:
\begin{equation}
\P^k_{h}(x_{h+1} \mid x_h, a_h)= \langle \phi_1(x_h, a_h), \mu_{h}^k (x_{h+1})\rangle,
\label{def: transition}
\end{equation}
\begin{equation}
    \E \big[ r^k_{h}(x_{h}, a_{h}) \mid x_{h}, a_{h} \big]= \langle\phi_0(x_h, a_h), \theta_{h}^0 \rangle + \langle\phi_1(x_h, a_h), \theta_{h}^k \rangle
\label{def: reward}
\end{equation}
for $k\in[K], h \in [H]$, where $x_h$ and $a_h$ are the state and action in the time $h$, respectively, $\mu_{h}^k$'s are unknown measures over $\Ssc$, and $\phi_0$ and $\phi_1$ are known feature maps. The feature maps can be thought of as the representation of relevant time-varying covariates. The effects of $\phi_0(x_h, a_h)$, which includes site-level covariates such as the size of the healthcare system, are assumed to be common across sites; while the effects of $\phi_1(x_h, a_h)$ are site-specific and hence capture the cross-site heterogeneity. Although the linear MDP model assumes that the transition kernel and expected rewards are the linear functions of the feature maps, the maps themselves can be nonlinear.

The proposed model allows the analysis of site-level information, which is helpful and sometimes necessary to learn optimal policies for sequential clinical decision-making \citep{gottesman2018evaluating}. For example, the hospital-level information, such as the number of intensive care units (ICU) and the ratio between doctors/nurses and patients, is related to the mortality of COVID-$19$ \citep{roomi2021declining}. Due to patient and site heterogeneity, personalized treatment is needed, and the optimal treatments can vary significantly among care units \citep{zhang2020individualized}. It is thus important to incorporate site-level information into policy optimization. We account for potential differences in the transition probability  \eqref{def: transition} across various sites by the site-specific measure $\mu_{h}^k$. The reward function defined in \eqref{def: reward} has two components. The first term $\langle\phi_0(x, a), \theta_{h}^0\rangle$ is homogeneous among the $K$ sites and can depend on the site-level features, while the second term  $\langle\phi_1(x, a), \theta_{h}^k\rangle$ is heterogeneous for each site.

Under this model, we propose a two-step Federated Dynamic Treatment Regime algorithm (FDTR) illustrated in Figure \ref{fig: ill} to estimate the model parameters and learn the optimal policy. First, the local dynamic treatment regime algorithm (LDTR) adapted from the pessimistic variant of the value iteration algorithm (PEVI) \citep{jin2020pessimism} is run at each site using its data only to learn the optimal policies and the corresponding value functions for each site. PEVI uses an uncertainty quantifier as the penalty function to eliminate spurious correlations from uninformative trajectories. Next, the summary statistics involving the learned value functions are shared across the sites. By utilizing these summary statistics, each site updates its policy using pessimism again. With the two steps, FDTR achieves efficient communication by sharing necessary summary statistics only once. 

\begin{figure}[t]
\begin{center}
\includegraphics[width=0.57\textwidth]{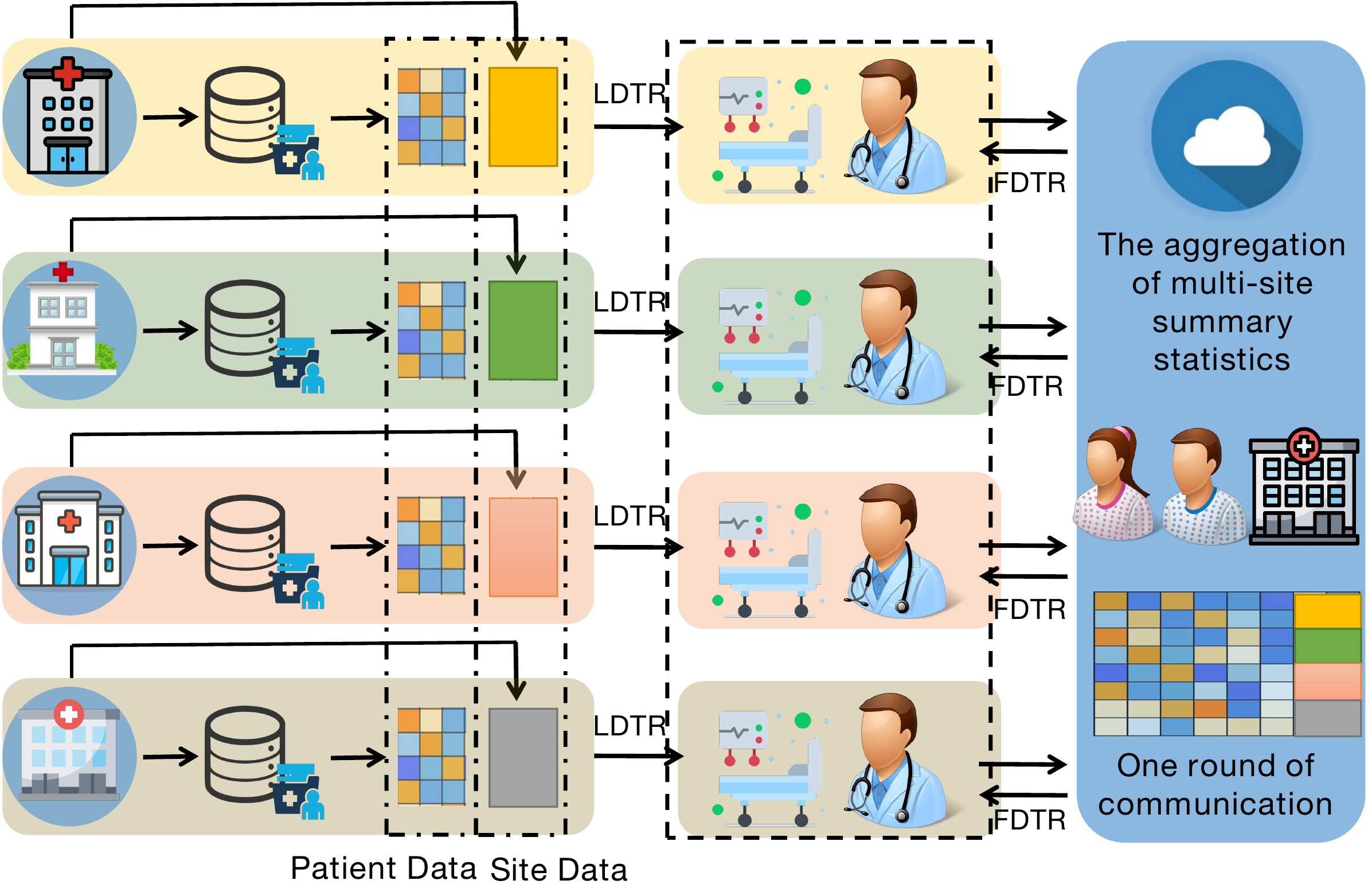}
\end{center}
\caption{An illustration of FDTR. 
\label{fig: ill}}
\end{figure}

\subsection{Related Work}

With single-source data, offline RL, such as Q-learning and PEVI, has been developed to estimate optimal DTR \citep{levine2020offline,jin2020pessimism}. Existing offline RL methods typically require strong assumptions on the collected data such as sufficient coverage \citep{fujimoto2019off, agarwal2020optimistic, gulcehre2020rl},  the ratio between the visitation measure of the target policy and behavior policies \citep{jiang2016doubly,thomas2016data,zhang2020gendice}, or a similar ratio over the state-action space \citep{NIPS2007_da0d1111,liu2019neural, scherrer2015approximate}. However, these assumptions either tend to be violated in healthcare settings or are not verifiable. 

One approach to overcome such challenges is to train optimal DTR using data from multiple sites. To further mitigate data-sharing constraints, federated learning has been developed to enable co-training algorithms based on shared summary-level data, with significant success in various tasks including prediction \citep{min2019predictive, hard2018federated}, and classification or clustering \citep{konevcny2016federated,mcmahan2017communication, yang2019federated}. { Compared to stochastic gradient descent (SGD)-based algorithms \citep{rothchild2020fetchsgd}, communication-efficient algorithms based on exchanging summary statistics are more practical for many healthcare problems since due to privacy concerns, only the summary-level data are ready for research and can be shared with researchers \citep{hong2021clinical}, while SGD-based algorithms still need to be performed based on patient-level data locally, which requires the researchers to have access to the patient-level data.} Communication-efficient algorithms with statistical guarantees have been developed for parametric likelihood framework \citep{jordan2018communication,battey2018distributed,duan2019heterogeneity}
and the modern machine learning models \citep{wang2019distributed, elgabli2020gadmm}.

In recent years, federated RL algorithms have also been proposed to co-train the underlying model overcoming data sharing constraints \citep{zhuo2019federated, nadiger2019federated,lim2020federated}. However, these algorithms focus on the online RL setting, which prevents them from being used for finding DTRs; additionally, many of them lack theoretical guarantees.  Until recently, several online algorithms have been proposed with guaranteed convergences \citep{fan2021fault,chen2022byzantine,xie2022fedkl,jadbabaie2022byzantine}. Federated RL is also different from multi-agent RL \citep{chen2018communication, zhang2019multi}, which considers the interaction of several agents in a common environment as illustrated by \cite{zhuo2019federated}. Nevertheless, 
limited methods exist for federated offline RL. Under homogeneous settings where the underlying models are assumed to be identical across sites, one may train local DTRs within each site and obtain a combined estimate by averaging the local DTR estimates across sites. However, such an approach cannot accommodate cross-site heterogeneity. As far as we are aware, no existing methods can perform co-training federated offline RL models in the presence of heterogeneity.

\subsection{Our Contributions}
To the best of our knowledge, FDTR is the first federated policy optimization algorithm for offline RL with sample complexity.
 Our primary innovation in the algorithm aspect is the multi-site MDP model, which takes site-level covariates into account to address heterogeneity and bias when performing federated learning. When site-level covariates are constant within each site, the PEVI algorithm cannot estimate site-level coefficients using single-site data. Although PEVI  could be applied to pooled patient-level data from all sites, this approach is not permissible under privacy constraints. To overcome this issue, we have developed an efficient two-step algorithm for federated learning in the presence of privacy constraints, using only summary-level data.

Our second innovation lies in our theoretical contributions. We provide theoretical justification for FDTR,
without the assumption of sufficient action coverage, a strong assumption which is often violated for clinical trials and observational healthcare data, due to the limited feasible treatments and sample size \citep{gottesman2018evaluating, gottesman2019guidelines}. Instead, we assume independence of the trajectories, a weaker assumption that is easier to be satisfied in reality within many disease contexts such as cancer and sepsis, as a standard simplification in theoretical development \citep{nachum2019dualdice,xie2021bellman,xie2021batch,zhan2022offline}. We provide an error rate for homogeneous coefficients that behaves as if all data were pooled together, which is particularly relevant for federated learning (Theorem \ref{theorem: estimation}). We offer an explicit rate (Corollary \ref{cor:coverage}) under the well-explored dataset assumption further and extend the linear MDP to a nonparametric model (Theorem \ref{theorem: spline}) to enhance the generalizability of the proposed method.

In the rest of the paper, we detail the multi-site MDP model in Section \ref{sec: problem}, present FDTR in Section \ref{sec: alg}, and show its theoretical property in Section \ref{sec: theorem}. 
Simulations are conducted in Section \ref{sec: simulation} and a real data application is given in Section \ref{sec: real}. 
Finally, discussions are shown in Section \ref{sec: conclusion}.

\section{Multi-site MDP Model}
\label{sec: problem}

We define $\| \cdot \|$ as $\ell_{2}$-norm, $\| \cdot \|_{\rm op}$ as operator norm. For vectors $w, \theta \in \R^d$, $\langle w,\theta\rangle = w\trans \theta$. For integer $d$, $[d]=\{1,\ldots, d\}$. For matrix $\bW$, $\bW^{+}$ is its Moore-Penrose inverse. $\bI_d$ denotes the $d \times d$ identity matrix. For symmetric matrices $\bA, \bB \in \R^{d\times d}$, we say $\bA \succeq \bB$ if $\bA - \bB$ is positive semi-definite. For $a, b \in \R$, $\min \{a,b\}^{+} = \max \big\{\min \{a,b\},0 \big\}$. We use $a_n = \widetilde O(b_n)$ when $a_n$ is bounded by $b_n$ up to logarithmic factors.

In this section, we detail the multi-site MDP model. Given the $k$th site for $k \in [K]$, a dataset $\Dsc_k = \big\{(x_{h}^{k,\tau}, a_{h}^{k,\tau}, r_{h}^{k,\tau})\big\}_{\tau, h=1}^{n_k,H}$ is collected a priori where at each step $h\in[H]$ of each trajectory (e.g., patient) $\tau \in [n_k]$, the experimenter (e.g., clinician) takes the action $a_h^{k,\tau} \sim \pi_h^k (a|x_h^{k,\tau})$ at the state $x_h^{k,\tau}$, receives the reward $r_h^{k,\tau} = r_h^k(x_h^{k,\tau}, a_h^{k,\tau})$ satisfying \eqref{def: reward}, and observes the next state $x_{h+1}^{k,\tau} \sim \P^k_h(\cdot \mid x_h=x_h^{k,\tau} ,a_h=a_h^{k,\tau})$ satisfying \eqref{def: transition}. { The  feature maps $\phi_0(x, a) \in \R^{d_0}$ and $\phi_1(x, a) \in \R^{d_1}$ in \eqref{def: transition} and \eqref{def: reward} are pre-specified. To guarantee sample complexity and model identifiability, we assume no co-linearity between $\phi_0(x, a)$ and $\phi_1(x, a)$ (detailed in Remark \ref{rmk5}). The transition probabilities only depend on features specified in $\phi_1(x, a)$ and we let $\phi_0(x, a)$ include site-level features and features with effects that are common across sites. For example, we can allow the age categories (e.g., children or adults) to have a common effect across sites, which are contained by $\phi_0$,  while allowing for a small degree of heterogeneity by including linear age effects in $\phi_1$.} All trajectories in $\Dsc_k$ for $k \in [K]$ are assumed to be independent. We impose no constraint on the behavior policies $\pi_h^k$'s and allow them to vary across the $K$ sites. For any policy $\pi=\{\pi_h\}_{h=1}^H$, we define the (state-)value function $V_h^{k,\pi}:\Ssc \to \R$ and the action-value function (Q-function) $Q_h^{k,\pi}:\Ssc \times \Asc \to \R$ for the $k$th site at each step $h\in[H]$ as
\begin{align}
 V_h^{k,\pi}(x) & = \E_{k,\pi}\Big[ \sum_{t=h}^H r^k_t(x_t, a_t)\mid x_h=x  \Big], 
\label{eq:def_value_fct} \\
Q_h^{k,\pi}(x,a) & = \E_{k,\pi}\Big[\sum_{t=h}^H r_t^k(x_t, a_t) \mid x_h=x, a_h=a  \Big].
\label{eq:def_q_fct}
\end{align}
Here the expectation  $\E_{k,\pi}$ in \eqref{eq:def_value_fct} and \eqref{eq:def_q_fct} is taken with respect to the randomness of the trajectory induced by $\pi$, which is obtained by taking the action $a_t \sim \pi_t (\cdot \mid x_t)$ at the state $x_t$ and observing the next state $x_{t+1} \sim \P^k_t(\cdot \mid x_t, a_t)$ at each step $t \in [H]$. 
Meanwhile, we fix $x_h = x \in \Ssc $ in \eqref{eq:def_value_fct} and $(x_h, a_h) = (x, a) \in \Ssc \times \Asc$ in \eqref{eq:def_q_fct}. Bellman equation implies 
$$
V_{h}^{k,\pi}(x)=\big\langle Q_{h}^{k,\pi}(x, \cdot), \pi_{h}(\cdot \mid x)\big\rangle_{\Asc}, \quad  Q_{h}^{k,\pi}(x, a) = (\B_{h}^k V_{h+1}^{k,\pi} )(x, a) 
$$
where $\langle\cdot, \cdot\rangle_{\Asc}$ is the inner product over $\Asc$, $\B_{h}^k$ is the  Bellman operator defined as $(\B_{h}^k f)(x, a)  = \E_k\big[r^k_{h}(x_{h}, a_{h})+f(x_{h+1}) \mid x_{h}=x, a_{h}=a\big]$ for any function $f: \mathcal{S} \rightarrow \mathbb{R}$, with $\E_k$ taken with respect to the randomness of the reward $r_{h}^k(x_{h}, a_{h})$ and next state $x_{h+1}$ where $x_{h+1} \sim  \P^k_{h}(x_{h+1}  \mid x_h, a_h)$. By \eqref{def: transition} and \eqref{def: reward}, for any function $V$, there exists $\bar \theta_{h}^k \in \R^{d_1}$ such that 
\begin{equation}
    (\B_{h}^k V)(x, a) = \langle\phi_0(x, a), \theta_{h}^0\rangle + \langle\phi_1(x, a),\bar \theta_h^k \rangle,
\label{eq: bellman}    
\end{equation}
where $\bar \theta_h^k = \theta_h^k +  \int_{x^{\prime} \in \mathcal{S}} \mu_{h}^k\left(x^{\prime}\right)  V( x^{\prime} ) \mathrm{d} x^{\prime}$. Therefore, the coefficients $\theta_{h}^0$ and $\bar \theta_h^k$ can be estimated through linear regression if the values $(\B_{h}^k V)(x, a)$ are known, which inspires us to derive the FDTR algorithm.

\section{Federated Dynamic Treatment Regimes Algorithm}
\label{sec: alg}
Suppose that we are treating the $k$th site. Inspired by \eqref{eq: bellman}, we notice the key step is to construct estimates $\widehat{V}_{h}^k$ of $V_{h}^k$ and $\widehat \B_{h}^k \widehat{V}^k_{h+1}$ of $\B_{h}^k V_{h}^k$  based on $\Dsc_k$ and the summary statistics of $ \{ \Dsc_j \}_{j\neq k}$. As mentioned in Section \ref{sec: intro}, pessimism plays an important role in the control of suboptimality. Define $\Dsc = \cup_{k=1}^K \Dsc_k$. Following the line of \cite{jin2020pessimism}, we achieve pessimism by the notion of multi-site confidence bound $\Gamma_{h}^k$ 
as follows.
\begin{definition}[Multi-site confidence bound] For the $k$th site, we say $\big\{\Gamma^k_{h}:  \mathcal{S} \times \Asc \rightarrow \mathbb{R} \big\}_{h=1}^{H}$ is a $\xi$-multi-site confidence bound of $V = \{ V_h \}_{h=1}^H$ with respect to $\mathbb{P}_{\Dsc}$ if the event
\begin{equation}
     \Esc_k (V) = \big\{|(\widehat{\B}^k_{h} V_{h+1})(x, a)-(\B^k_{h} V_{h+1})(x, a)| \leq \Gamma_{h}^k(x, a) \text { for all }(x, a) \in \mathcal{S} \times \Asc, h \in[H]\big\}
\label{eq:def_event_eval_err_general}
\end{equation}
satisfies $\P_{\Dsc}\big(\Esc_k (V) \big) \geq 1-\xi$. Here the value functions $V = \{ V_h \}_{h=1}^H$ and $\big\{\Gamma^k_{h}\big\}_{h=1}^{H}$ can depend on $\Dsc$. Specifically, if both of them only depend on $\Dsc_k$, then $\P_{\Dsc}\big(\Esc_k (V) \big) = \P_{\Dsc_k}\big(\Esc_k (V) \big)$.
\label{def1}
\end{definition}
By definition,  $\Gamma_{h}^k$ quantifies the approximation error of $\widehat{\B}^k_{h} V_{h+1}$ for ${\B}^k_{h} V_{h+1}$, which is important in eliminating the spurious correlation as discussed by \cite{jin2020pessimism}. Before we present the details of FDTR, we introduce the following notations for simplicity:
\begin{equation}
\begin{aligned}
&  \Phi_{j,h}^k = \big( \phi_j(x_{h}^{k,1}, a_{h}^{k,1}), \ldots, \phi_j(x_{h}^{k,n_k}, a_{h}^{k,n_k}) \big)\trans, j =0,1 ;\\
&  \Phi_h^k = \big( \phi^k(x_{h}^{k,1}, a_{h}^{k,1}), \ldots, \phi^k(x_{h}^{k,n_k}, a_{h}^{k,n_k}) \big)\trans;  \Phi_h = \big( (\Phi_h^1)\trans,\ldots, (\Phi_h^K)\trans  \big) \trans; \Lambda_h =  \Phi_h\trans  \Phi_h; \\
& Y_h^k(V) = \big(r_h^{k,1} + V( x^{k,1}_{h+1}) ,\ldots,r_h^{k,n_k} + V( x^{k,n_k}_{h+1}) \big)\trans;  \Lambda_h^k = (\Phi_{0,h}^k,\Phi_{1,h}^k)\trans (\Phi_{0,h}^k,\Phi_{1,h}^k); \\
\end{aligned}    
\label{def: notation}
\end{equation}
where $\phi^k(x,a)= \big( \phi_0 (x, a)\trans, \bzero_{(k-1)d_1}\trans, \phi_1 (x, a)\trans, \bzero_{(K-k) d_1}\trans \big)\trans \in \R^{d_K}$ with $d_K = d_0 + d_1 K$, and $V:\Ssc \to \R$ is some value function. Besides, let $\phi(x,a) = ( \phi_0 (x, a)\trans, \phi_1 (x, a)\trans )\trans \in \R^d$ where $d = d_0 + d_1$ and  $N = \sum_{k=1}^K n_k$. We define the empirical mean squared Bellman error:
\begin{equation}
 M_{h}^k(\theta^0,\theta^k \mid \widehat V_{h+1}^k) = \big\| Y_h^k(\widehat V_{h+1}^k) - \Phi_{0,h}^k \theta^0 - \Phi_{1,h}^k \theta^k \big \|^2 \text{ for } h \in [H], k \in [K] \,.
\end{equation}
If the individual-level information is shareable across the $K$ sites, we can aggregate $M_{h}^k(\theta^0,\theta^k \mid \widehat V_{h+1}^k)$ to estimate $\theta^0_h$ and $\bar \theta_h^k$  for $k \in [K]$ simultaneously. However, as we state above, it is usually not possible in reality due to privacy concerns. On the contrary, we assume that we have some preliminary estimators for the value functions $\big\{ \widetilde{V}^k_{h} \big\}_{k,h=1}^{K,H}$. Here we adopt the PEVI of linear MDP algorithm  in \cite{jin2020pessimism} to obtain the preliminary estimators  $\big\{ \widetilde{V}^k_{h} \big\}_{k,h=1}^{K,H}$, which is summarized in Algorithm \ref{alg0}.
\begin{algorithm}[t]
\caption{Local Dynamic Treatment Regime (LDTR).}\label{alg0}
\begin{algorithmic}[1]

\For{ site $k = 1,2,\ldots, K$}

\State Input: Dataset $\Dsc_k =\big\{(x_{h}^{k,\tau}, a_{h}^{k,\tau}, r_{h}^{k,\tau})\big\}_{\tau, h=1}^{n_k, H}$.

\State Set $\lambda=1, \quad \alpha_k = c d H \sqrt{\zeta_k}, \quad \text { where } \zeta_k =\log (2 d H n_k / \xi) $.

\State Initialization: set $\widetilde{V}_{H+1}^k(x) \leftarrow 0$. 
 
\For{ step $h = H,H-1,\ldots, 1$}
    \State Set $\Lambda_{h}^k$ as in \eqref{def: notation} and $\widetilde \theta_h^k = (\Lambda_{h}^k + \lambda \bI_d )^{-1}\big( Y^k_h(\widetilde V_{h+1}^{k})\trans \Phi_{0,h}^k, Y^k_h(\widetilde V_{h+1}^{k})\trans \Phi_{1,h}^k \big)\trans $.
      
    \State Set $\widetilde \Gamma_{h}^k(\cdot, \cdot) \leftarrow \alpha_k  \sqrt{ \phi(\cdot, \cdot)\trans (\Lambda_{h}^k + \lambda \bI_d )^{-1} \phi(\cdot, \cdot) }$.
    
    \State Set $ \widetilde{Q}^k_{h}(x, a) \leftarrow \min \big\{\phi(x, a)\trans \widetilde \theta_h^{k} - \widetilde \Gamma_{h}^k(x, a), H-h+1\big\}^{+}$.
     
    \State Set $ \widetilde{\pi}_{h}^k(\cdot \mid \cdot) \leftarrow \arg \max _{\pi_{h}}\big\langle\widetilde{Q}^k_{h}(\cdot, \cdot), \pi_{h}(\cdot \mid \cdot)\big\rangle_{\Asc}$.
      
     \State Set $\widetilde{V}^k_{h}(\cdot) \leftarrow\big\langle\widetilde{Q}^k_{h}(\cdot, \cdot), \widetilde{\pi}_{h}^k(\cdot \mid \cdot)\big\rangle_{\Asc}$.
    \EndFor
    \State Output: $\widetilde \pi^k = \{\widetilde \pi_h^k  \}_{h=1}^H$, $\widetilde{V}^k = \{ \widetilde{V}_h^k  \}_{h=1}^H$.
 \EndFor
\end{algorithmic}
\end{algorithm}
We consider the objective function for $\theta = \big( (\theta^0)\trans, (\theta^1)\trans,\ldots, (\theta^K)\trans \big)\trans \in \R^{d_K}$: 
\begin{equation}
  f_h^k(\theta \mid \widehat V^k, \{\widetilde V^j\}_{j \neq k})  =  \sum_{j\neq k} M_{h}^j(\theta^0,\theta^j \mid \widetilde V_{h+1}^j)  +  M_{h}^k(\theta^0,\theta^k \mid \widehat V_{h+1}^k)  + \lambda( \|\theta_0 \|^2 + \|\theta_k \|^2) \,.
\label{eq:obj}
\end{equation}
 For the $k$th site, we only need $\theta^0$ and $\theta^k$ to estimate the value functions. We do not impose the $\ell_2$ regularization on $\theta_j$ for $j \neq k$ for the simplicity of theoretical analysis, which can be revealed in the proof of Theorem \ref{theorem: main}.  The objective function \eqref{eq:obj} has the explicit minimizer 
\begin{equation}
\begin{aligned}
     \widehat \theta_{h,k} & =  \big( (\widehat \theta_{h,k}^{0})\trans,\ldots, (\widehat \theta_{h,k}^K)\trans  \big)\trans  
     & = (\Lambda_h + \bH_{k,\lambda}) ^{+}\{ \sum_{j\neq k}  (\Phi_h^j)\trans Y_h^j(\widetilde V_{h+1}^j)  + (\Phi_h^k)\trans Y_h^k(\widehat V_{h+1}^k) \},
\end{aligned}
\label{def: hat beta}    
\end{equation}
where $\bH_{k,\lambda} = {\rm diag}(\lambda \bone_{d_0} ,\bzero_{d_1(k-1)},\lambda \bone_{d_1},\bzero_{d_1(K-k)})$. Let $\bP_{\Phi_{1,h}^j} = \Phi_{1,h}^j \big( (\Phi_{1,h}^j)\trans \Phi_{1,h}^j \big)^{+} (\Phi_{1,h}^j)\trans \in \R^{n_j \times n_j}$ be the projection matrix to the column space of $\Phi_{1,h}^j$. We define
\begin{equation}
\Sigma_h^{k}  = \Lambda_h^{k} + \lambda \bI_d +     \begin{bmatrix}
 \sum_{j \neq k}^{K}  ( \Phi_{0,h}^j)\trans \big( \bI_{n_j} - \bP_{\Phi_{1,h}^j} \big)\Phi_{0,h}^j  & \mathbf{0}_{d_0 \times d_1}\\
\mathbf{0}_{d_1 \times d_0}  & \mathbf{0}_{d_1 \times d_1}
\end{bmatrix} \in \R^{d \times d} \,.
\label{def: sigma}
\end{equation}
We then set 
$(\widehat \B^{k}_{h} \widehat V_{h+1}^k)(x,a) = \phi^k(x,a)\trans \widehat \theta_{h,k} = \phi_0 (x,a)\trans   \widehat \theta_{h,k}^0+ \phi_1 (x,a)\trans \widehat \theta_{h,k}^k \,.$
Meanwhile, we construct $\Gamma^k_{h}$ based on $\Dsc$ as
\begin{equation}
  \widehat \Gamma_{h}^k(x, a) 
  = \alpha \sqrt{ \phi(x, a)\trans ( \Sigma_{h}^k)^{-1} \phi(x, a)}
\label{def: gamma} 
\end{equation}
at each step $h \in[H]$. 
Here $\alpha > 0$ is a scaling parameter to be specified later according to the theoretical rate. In addition, we construct $\widehat{V}^k_{h}$ based on $\Dsc$
as
$$
\begin{aligned}
& \widehat{Q}^k_{h}(x, a) =\min \{(\widehat{\B}^k_{h} \widehat{V}^k_{h+1})(x, a)- {  \widehat \Gamma^k_{h}(x, a) } , H-h+1 \}^{+}, \\
& \widehat{V}^k_{h}(x) =\langle\widehat{Q}^k_{h}(x, \cdot), \widehat{\pi}^k_{h}(\cdot \mid x)\rangle_{\Asc}, \quad \text { where } \widehat{\pi}^k_{h}(\cdot \mid x) = \underset{\pi_{h}}{\arg \max }\langle\widehat{Q}^k_{h}(x, \cdot), \pi_{h}(\cdot \mid x)\rangle_{\Asc}.
\end{aligned}
$$
Notice that the above procedures only require summary statistics from other sites. To be specific, only $(\Phi^j_h)\trans\Phi^j_h$ and $(\Phi^j_h)\trans Y^j_h(\widetilde V_{h+1}^j)$, for $j \in [K]/\{k\}$ are required when we are treating the $k$th site. 
The FDTR algorithm is summarized in Algorithm \ref{alg1}. The communication cost is $O(K d^2 H)$ for getting the summary statistics from the other sites and the computational complexity is $O(N d^2 H  + K d^3 H)$ for calculating the summary statistics and conducting $KH$ linear regressions. For the homogeneous parameter $\theta_{h}^0$, one may use the average estimators from the $K$ site defined as $\widehat \theta_{h}^0 = \frac{1}{K}\sum_{k=1}^K \widehat \theta_{h,k}^0$ in practice, while the rate of $\widehat \theta_{h}^0$ is the same as $\theta_{h,k}^0$ for each $k$ as revealed by Corollaries \ref{cor:coverage} and \ref{cor:site-spec}.

\begin{remark}
In comparison, using SGD for linear regression necessitates multiple communication rounds, each costing $O(d H K)$ due to gradient exchanges across $K$ sites for $H$ time points. While SGD demands $O(d H^2/\epsilon^2)$ iterations for an 
$\epsilon$-optimal solution in convex problems \citep{harold1997stochastic},  leading to a total computational cost of $O(d^2H^3K/\epsilon^2)$ among sites, FDTR, assuming $n_k > d$, is more time-efficient than SGD when $\epsilon = o(H/\sqrt{n_k})$,  a frequent scenario in RL \citep{agarwal2021theory}.
\end{remark}

\begin{algorithm}[t]
\caption{  Federated Dynamic Treatment Regime (FDTR): for the $k$th site. }\label{alg1}
\begin{algorithmic}[1]
\State Input: Dataset $\Dsc_k=\big\{(x_{h}^{k,\tau}, a_{h}^{k,\tau}, r_{h}^{k,\tau})\big\}_{\tau, h=1}^{n_k, H}$. The summary statistics from the other $K-1$ sites $\big \{ (\Phi^j_h)\trans\Phi^j_h, (\Phi^j_h)\trans Y^j_h(\widetilde V_{h+1}^j) \big\}_{h=1}^{H}, \text{ for } j \in [K]/\{k\}.$
\State Set 
$\lambda=1, \quad \alpha = c  d H \sqrt{\zeta}, \quad \text { where } \zeta=\log (2 d H N / \xi) $.

\State Initialization: set $\widehat {V}_{H+1}^k(x) \leftarrow 0$.


\For{ step $h = H,H-1,\ldots, 1$}
 
   \State Set $ \big( (\widehat \theta_{h,k}^0)\trans, (\widehat \theta_{h,k}^k)\trans \big)\trans$  as in \eqref{def: hat beta} and $\widehat \Gamma_{h}^k(\cdot, \cdot)$ as in \eqref{def: gamma}.
    
     \State Set $ \widehat{Q}^k_{h}(\cdot, \cdot) \leftarrow \min \big\{\phi(\cdot, \cdot)\trans \big( (\widehat \theta_{h,k}^0)\trans, (\widehat \theta_{h,k}^k)\trans \big)\trans - \widehat \Gamma_{h}^k(\cdot, \cdot), H-h+1\big\}^{+}$.
    
    
     
    
    
     
     
    \State Set $ \widehat{\pi}_{h}^k(\cdot \mid \cdot) \leftarrow \arg \max _{\pi_{h}}\big\langle\widehat{Q}^k_{h}(\cdot, \cdot), \pi_{h}(\cdot \mid \cdot)\big\rangle_{\Asc}$.
      
     \State Set $\widehat{V}^k_{h}(\cdot) \leftarrow\big\langle\widehat{Q}^k_{h}(\cdot, \cdot), \widehat{\pi}_{h}^k(\cdot \mid \cdot)\big\rangle_{\Asc}$.
\EndFor
\State Output: $\widehat \pi^k = \{\widehat \pi_h^k  \}_{h=1}^H$.
\end{algorithmic}
\end{algorithm}

\section{Theoretical Analysis}
\label{sec: theorem}
We now state the theoretical properties of FDTR.
For the multi-site  MDP $(\mathcal{S}, \Asc, H, \mathcal{P}^k, r^k)$, we use $\pi^{k,*}, Q_{h}^{k,*}$, and $V_{h}^{k,*}$ to denote the optimal policy, Q-function, and value function for the $k$th site, respectively. We have $V_{H+1}^{k,*}=0$ and the Bellman optimality equation
$
V_{h}^{k,*}(x)=\max _{a \in \Asc} Q_{h}^{k,*}(x, a)$, $Q_{h}^{k,*}(x, a)=(\B_{h}^k V_{h+1}^{k,*})(x, a)  \, .
$
Meanwhile, the optimal policy $\pi^{k,*}$ is specified by
$
\pi_{h}^{k,*}(\cdot \mid x)=\underset{\pi_{h}}{\arg \max }\big\langle Q_{h}^{k,*}(x, \cdot), \pi_{h}(\cdot \mid x)\big\rangle_{\Asc}$ and $V_{h}^{k,*}(x)=\big\langle Q_{h}^{k,*}(x, \cdot), \pi_{h}^{k,*}(\cdot \mid x)\big\rangle_{\Asc} \, ,
$
where the maximum is taken over all functions mapping from $\mathcal{S}$ to distributions over $\Asc$. We aim to learn a policy that maximizes the expected cumulative reward. Correspondingly, we define the performance metric at the $k$th site as $\operatorname{SubOpt}^k(\pi ; x)=V_{1}^{\pi^{k,*}}(x)-V_{1}^{k,\pi}(x)$, 
which is the suboptimality of the policy $\pi$ with regard to  $\pi^{k,*}$ given the initial state $x_{1}=x$.

We are presenting both data-dependent and explicit rate results. The data-dependent results are more adaptive to the concrete realization of the data and require fewer assumptions, while the explicit rate results are more straightforward to compare and give more insights into how the suboptimality and the estimators depend on $n_k$, $d$, and $K$. For technical simplicity, we assume that $\|\phi_0(x, a)\|^2  +  \|\phi_1(x, a)\|^2 \leq 1$ for all $(x, a) \in \mathcal{S} \times \Asc$, $\|\theta_{h}^0\| \leq \sqrt{d_0}$,  $\max_{k \in [K], h\in[H]} \left\{\left\|\mu^k_{h}(\mathcal{S})\right\|,\|\theta_{h}^k\|\right\} \leq \sqrt{d_1}$, which can be guaranteed after suitable normalization, where with abuse of notation, we define
$\left\|\mu^k_{h}(\mathcal{S})\right\|=\int_{\mathcal{S}}\left\|\mu^k_{h}(x)\right\| \mathrm{d} x$. Theorems
\ref{theorem: prelim} and \ref{theorem: main} characterize the data-dependent suboptimality of Algorithms \ref{alg0} and \ref{alg1}. 

\begin{theorem}[Suboptimality of the Preliminary Estimators]
\label{theorem: prelim}
In Algorithm \ref{alg0}, we set $\lambda=1$, $\alpha_k = c dH\sqrt{\zeta_k}$, where $\zeta_k = \log(2dH n_k/\xi)$,
$c>0$ is an absolute constant and $\xi \in (0,1)$ is the confidence parameter. Then $\{\widetilde \Gamma^k_h\}_{h=1}^H$ in Algorithm \ref{alg0} is a $\xi$-multi-site confidence bound of $\widetilde V^k = \{\widetilde V_h^k  \}_{h=1}^H$. For any $x \in \Ssc$, $\widetilde \pi^k = \{ \widetilde \pi_h^k \}_{h=1}^H$ in  Algorithm \ref{alg0} satisfies $\operatorname{SubOpt}^k\big(\widetilde \pi^k ;x \big) \leq 2 \alpha_k \sum_{h=1}^H\E_{k,\pi^{k,*}}\Bigl[ \sqrt{ \phi(x_h,a_h)\trans (\Lambda_h^k + \lambda \bI_d)^{-1}\phi(x_h,a_h)} \mid x_1=x\Bigr]$
with probability at least $1-\xi$, for any $k \in [K]$. 
Here $\E_{k, \pi^{k,*}}$ is taken with respect to the trajectory induced by $\pi^{k,*}$ in the underlying MDP given the fixed  $\Lambda_h^k$. 
\end{theorem}

\begin{theorem}[Suboptimality of FDTR]
\label{theorem: main}
In Algorithm \ref{alg1}, we set
\begin{equation}
  \lambda=1,\quad \alpha = c dH\sqrt{ \zeta}, \quad \text{where~~}\zeta= \log(2dHN/\xi) \,.
\label{tuning}
\end{equation}
Here $c>0$ is an absolute constant and $\xi \in (0,1)$ is the confidence parameter. Then $\{\widehat \Gamma^k_h\}_{h=1}^H$  specified in \eqref{def: gamma} is a $\xi$-multi-site confidence bound of $\widehat V^k = \{\widehat V_h^k  \}_{h=1}^H$ in Algorithm \ref{alg1}. Besides,  for any $x \in \Ssc$, with probability at least $1-\xi$, for any $k \in [K]$, $\widehat \pi^k = \{ \widehat \pi_h^k \}_{h=1}^H$ in  Algorithm \ref{alg1} satisfies \begin{equation}
\operatorname{SubOpt}^k\big(\widehat \pi^k ;x \big) \leq 2 \alpha \sum_{h=1}^H\E_{k,\pi^{k,*}}\Bigl[\sqrt{\phi(x_h,a_h)\trans ( \Sigma_h^k)^{-1}\phi(x_h,a_h) }\mid x_1=x\Bigr] \,.
\label{eq:main rate}
\end{equation}
Here $\E_{k, \pi^{k,*}}$ is taken concerning the trajectory induced by $\pi^{k,*}$ in the underlying MDP given the fixed matrix $\Sigma_h^k$. 
\end{theorem}
\begin{remark}
     The theoretical results of FDTR are structural, which means that even if the local estimator were to be altered from PEVI to another choice such as the standard value iteration algorithm (i.e., setting $\widetilde \Gamma_h^k(x,a) = 0$ in Algorithm \ref{alg0}),
    analogous theoretical results for FDTR would still be applicable. Our examination of the proof of Theorem \ref{theorem: main} elucidates that FDTR's theoretical outcomes are maintained provided the initial estimates are encompassed within the function class delineated in (S2) of Supplementary S2.  
\end{remark}
By the definition of $\Sigma_h^k$ in \eqref{def: sigma}, it is easy to show that $\Sigma_h^k \succeq \Lambda_h^k + \lambda \bI_d$, which implies 
that  the right-hand side of \eqref{eq:main rate} is finite irrespective of the behavior policy employed and  
$\phi(x_h,a_h)\trans (\Sigma_h^k)^{-1}\phi(x_h,a_h) \leq \phi(x_h,a_h)\trans ( \Lambda_h^k + \lambda \bI_d)^{-1}\phi(x_h,a_h)$.  
Comparing  Theorem \ref{theorem: prelim} and  Theorem \ref{theorem: main}, the scaling parameters $\alpha$ and $\alpha_k$ only have a difference in the logarithm term, while Theorem \ref{theorem: main} has a sharper bound for the expected term, which is owing to the utilization of multi-site information. 
\begin{theorem}[Estimation error for the homogeneous effects]
\label{theorem: estimation}
For $\alpha$ specified in Theorem \ref{theorem: main}, $\widehat \theta_{h,k}^0$ satisfies $ \|\widehat \theta_{h,k}^0 - \theta_h^0 \|  \leq  \alpha  \| (\Sigma_h^k)^{-1}\|_{\rm op}^{1/2}$ for any $k \in [K]$ with probability $1-\xi$. \end{theorem}
Recall that $\theta_h^0$ is not estimable when $\phi_0(x,a)$ is a constant vector for the trajectories in the same site (i.e., $\phi_0(x,a)$ only contains the site-level covariates), since $(\Phi_{0,h}^k)\trans \Phi_{0,h}^k = n_k \phi_0(x_h^{k,\tau},a_h^{k,\tau}) \phi_0(x_h^{k,\tau},a_h^{k,\tau})\trans$ for $\tau \in [n_k]$ is a rank-one matrix. However, when we combine $K$ sites, if $K \geq d_0$, under mild conditions, $\sum_{k=1}^K (\Phi_{0,h}^k)\trans \Phi_{0,h}^k$ will be a full rank matrix and $\theta_h^0$ can be estimated with theoretical guarantee, whose details will be given later.

The above theorems are data-dependent. To present the explicit rates related to the sample sizes $n_k, k \in [K]$, we need to impose more assumptions on the data generation mechanism. For notational simplicity, we define the following (uncentered) covariance matrices
\begin{align}
    \label{eq:cov-var}
    \Csc^k_{ij, h} = \E_{k}\bigl[\phi_i(x_h, a_h)\phi_j(x_h, a_h)\trans\bigr] \text{ and } \Csc^k_{h}= \E_{k}\bigl[\phi(x_h, a_h)\phi(x_h, a_h)\trans\bigr],
\end{align}
for $i,j= 0, 1$, $h \in [H]$, and $k \in [K]$. By definition, $\Csc^k_{00, h}$ is the covariance of homogeneous features, $\Csc^k_{11, h}$ is the covariance of heterogeneous features and $\Csc^k_{h}$ is their joint covariance matrix. 
In what follows, we illustrate the necessity of FDTR by considering the setting where the data-collecting process well explores the state-action space. Recall that $N = \sum_{k = 1}^K n_k$ is the number of trajectories of all the sites. We have the following corollary.
\begin{corollary}[Suboptimality of FDTR with Well-Explored Dataset]
\label{cor:coverage}
Assume that there exists an absolute constant $b > 0$ such that $\Csc^k_h \succeq b \bI_d$ for any $h \in [H]$ and $k \in [K]$ and $\| \phi_0(x,a)\|^2 \leq d_0 / d$ and $\|\phi_1(x, a)\|^2 \le d_1 / d$. If we choose the tuning parameters as \eqref{tuning}, it holds with probability at least $1 - 2 \xi$ that $\operatorname{SubOpt}^k\big(\widehat \pi^k ;x \big)  \le 2 \sqrt{2} \alpha H \sqrt{ \frac{1}{b d } \Big(
    \frac{d_0}{N } +  \frac{d_1}{n_k} \Big) }$ and $\|\widehat \theta_{h,k}^0 - \theta_h^0 \|  \leq \frac{\sqrt{2} \alpha}{\sqrt{N b}}$. 
Accordingly, we have $\widehat \theta_{h}^0 = \frac{1}{K}\sum_{k=1}^K \widehat \theta_{h,k}^0$ satisfying  
$\|\widehat \theta_{h}^0 - \theta_h^0 \|  \leq \frac{\sqrt{2} \alpha}{\sqrt{N b}}$.
\end{corollary}

{
\begin{remark}
    Under the well-coverage assumption in Corollary \ref{cor:coverage}, we can get similar results for LDTR. Specifically, $\operatorname{SubOpt}^k\big(\widetilde \pi^k ;x \big)  \le 2 \sqrt{2} \alpha H \sqrt{ \frac{1}{b d } \Big(
    \frac{d_0}{n_k } +  \frac{d_1}{n_k} \Big) }$ and $ \|\widetilde \theta_{h,k}^0 - \theta_h^0 \|  \leq \frac{\sqrt{2} \alpha}{\sqrt{n_k b}}$  with probability at least $1 - 2 \xi$,  
where $\widetilde \theta_{h,k}^0$ is the estimated preliminary homogeneous coefficient corresponding to the first $d_0$ elements of $\widetilde \theta_{h,k}$ defined in line 6 of Algorithm \ref{alg0}. 
\end{remark} }

\begin{remark}
\label{rmk5}
Notice that $\|\phi_0(x, a)\|^2  +  \|\phi_1(x, a)\|^2 \leq 1$. The additional assumption that $\| \phi_0(x,a)\|^2 \leq d_0 / d$ and $\|\phi_1(x, a)\|^2 \le d_1 / d$ implies that the norm is uniformly distributed among $\phi_0$ and $\phi_1$, which is only needed to obtain the explicit rate dependent on $d_0, d_1$ and $d$. Otherwise, the upper bound of $\operatorname{SubOpt}^k\big(\widehat \pi^k ;x \big)$ can be replaced by $$ 2 \sqrt{2} \alpha H \sqrt{\|\phi_0(x, a)\|^2/(\|\phi(x, a)\|^2 b N) + \| \phi_1(x, a)\|^2/(\|\phi(x, a)\|^2b n_k)}\,.$$
Besides, the assumption that $\cC_h^k \geq b \bI_d$ implies that $\phi_0(x,a)$ and $\phi_1(x,a)$ can not be fully linear dependent, that is, a feature exists in both $\phi_0$ and $\phi_1$. Otherwise, $\cC_n^k$ may not be full rank, violating the lower bounded eigenvalue assumption. However, the features in $\phi_0$ and $\phi_1$ can be correlated, as the age covariate example in Section \ref{sec: problem}. 
\end{remark}

The assumption that $\Csc^k_h \succeq b \bI_d$ assumes that the dataset well explores the state and action space \citep{duan2020minimax, jin2020pessimism} in the sense that each direction in the feature space is well explored in the dataset $\Dsc_k$. When considering the well-explored dataset, Corollary \ref{cor:coverage} shows that FDTR attains the suboptimality of order $\tilde  O\big( H^2 d \sqrt{d_0/(d N) + d_1/ (d n_k)} \big)$, while the suboptimality of the preliminary estimator is $\tilde O( H^2 d /\sqrt{ n_k} )$ \citep{jin2020pessimism}. { So FDTR will have a tighter rate than LDTR when $d_1 \ll d$.} Thus, FDTR is more efficient than LDTR in eliminating the suboptimality arising from the site-level features.  In specific, when $d_1 = o(d/K)$,  the suboptimality of FDTR is $O(H^2 d\sqrt{1 / N})$  which achieves the desirable rate of $N^{-1/2}$. Given our assumption of heterogeneity across the $K$ sites as outlined in models \eqref{def: transition} and  \eqref{def: reward}, only the homogeneous component can be universally shared among sites, boasting an effective sample size of $N = \sum_k n_k$. In contrast, the heterogeneous component maintains an effective sample size of $n_k$, precisely mirroring the rate of $\widehat \pi^k$ as defined in Corollary \ref{cor:coverage}. Besides, in Corollary \ref{cor:coverage}, we further show that the statistical rate of the homogeneous parameter $\theta_h^0$ is $\tilde  O\big( H d N^{-1/2} )$ even if $d_1 \asymp d$, which is better than the rate of LDTR with a factor of $\sqrt{K}$. When $H$ and $d$ are fixed, it achieves the standard parametric rate $O(N^{-1/2})$ which performs like pooling all data together. 

\begin{remark}
Consider the setting where the dataset well explores the state-action space, which recovers the online setting after the exploration. For example, when $d_0/d \approx 1$ and $d_1/d \approx 0$, our setting recovers the homogeneity setting where the data of all the sites follow the same distribution. In such a setting, our sample complexity for achieving the $\epsilon$-optimal solution is at the order of $O(1/\epsilon^2)$, which matches the rate of \cite{fan2021fault}.   Given that our estimators reach these optimal rates (up to logarithmic factors), there is no statistical advantage to permitting multiple communication rounds. 
\end{remark}

However, the assumption that $\Csc^k_h \succeq b \bI_d$ will be violated when ${\rm rank} ( \Csc^k_{00,h} ) < d_0$, which happens when $d_0 > 1$ and $\phi_0$ is constant in the same site. To this end, $\theta_h^0$ cannot be estimated by using the single site data only. However, $\theta_h^0$ can still be estimated by FDTR as shown by the following corollary. 
\begin{corollary}[Suboptimality of FDTR with Homogeneous Covariates] 
\label{cor:site-spec}
We assume that there exists an absolute constant $b > 0$ such that $\Csc_{11, h}^k \succeq b \bI_{d_1}$ and 
\begin{equation}
 \lambda_{\min} \Big(\frac{1}{K} \sum_{k =1}^K \Csc_{00, h}^k \Big) \geq \max\{ b, 2 b^{-1} \|\Csc_{10, h}^k\|^2 \} \quad \forall k \in [K],  h \in [H] \,,
 \label{god}
\end{equation}
and $\| \phi_0(x,a)\|^2 \leq d_0 / d$ and $\|\phi_1(x, a)\|^2 \le d_1 / d$, $\forall (x,a) \in \Ssc \times \Asc$. If we choose the tuning parameters as \eqref{tuning}, it holds with probability at least $1 - 2 \xi$ that $ \operatorname{SubOpt}^k\big(\widehat \pi^k ;x \big)  \le 2 \sqrt{2} \alpha H  \sqrt{ \frac{1}{b d } \Big(
    \frac{d_0}{N} +  \frac{d_1}{n_k} \Big) }$ and $\|\widehat \theta_{h,k}^0 - \theta_h^0 \|  \leq \frac{\sqrt{2} \alpha}{\sqrt{N b}}$. Accordingly, we have   
$\|\widehat \theta_{h}^0 - \theta_h^0 \|  \leq \frac{\sqrt{2} \alpha}{\sqrt{N b}}$.
\end{corollary}
By imposing the assumption in \eqref{god}, we get the same suboptimality and statistical rate as in Corollary \ref{cor:coverage}.
In particular, the assumption in \eqref{god} allows for ${\rm rank} ( \Csc^k_{00,h} ) < d_0$ for a single site $k$, which allows $\phi_0$ to be constant within each site. Recall that $\Csc^k_{h}$ is the (uncentered) covariance between the homogeneous feature $\phi_0$ and the heterogeneous feature $\phi_1$. The inequality in \eqref{god} holds when such a correlation is relatively small. Besides, we require $\lambda_{\min} (K^{-1} \sum_{k =1}^K \Csc_{00, h}^k)$ to be relatively large, which holds when the space of site-level features is well-explored by the dataset $\Dsc$.

 \subsection{Extension to Non-parametric Estimation}

Theorems \ref{theorem: prelim} and \ref{theorem: main} rely on the linear MDP assumption, which is restrictive in many real-world applications. However, FDTR is robust to model misspecified when the linear MDP structure is violated. As long as the discrepancy is not significant, FDTR still enjoys theoretical guarantees as shown in Supplementary S8. To further enhance the generalizability of FDTR, we extend it to non-parametric estimation. Specifically, instead of imposing the parametric assumptions 
\eqref{def: transition} and \eqref{def: reward} on the MDP, we consider using basis expansions with the number of basis functions growing with sample size to approximate the transition kernels and reward functions. Assume that $x_h = (x_{0h}\trans,x_{1h}\trans)\trans$ where $x_{0h}$ is the homogeneous feature and  $x_{1h}$ is the heterogeneous feature, and $x_{0h}, x_{1h}, a_h \in \R^{m_0}, \R^{m_1}, \R^{m_2}$, respectively, and let $m = m_0 + m_1 + m_2$. For simplicity, we assume that the homogeneous features in $\phi_0$ remain constant during the whole horizon, and $\Asc$ and $\Ssc$ are bounded, which are common for many real-world applications. We also assume that $\P_k^k(x_{0(h+1)} \mid x_h, a_h)$ and $\E \big[ r^k_{h}(x_{h}, a_{h}) \mid x_{h} = x, a_{h} =a \big]$ have $q$th order continuous and bounded derivative with respect to $x_{0(h+1)}$, $x_h$ and $a_h$. Let $\{\psi_{0s}( x_{0h} ,a_h), s \in \N^+\}$ and $\{\psi_{1 s}( x_{1h} ,a_h), s \in \N^+\}$ be the basis functions of $(x_{0h},a_h)$ and $(x_{1h},a_h)$, respectively. By the property of multivariate basis expansion \citep{hastie2009elements}, the basis functions of $(x_{1h},a_h,x_{1(h+1)})$ can be represented by $\{\psi_{1 s}(x_{1h} ,a_h)\psi_{2 t}( x_{1(h+1)} ), s,t \in \N^+\}$, where $\{ \psi_{2 t}( x_{1(h+1)} ), t \in \N^+\}$ are the basis functions of $x_{1(h+1)}$. We define $\phi_{0}(x_h,a_h) = \big( \psi_{01}(x_{0h},a_h), \ldots, \psi_{0 d_0}(x_{0h},a_h) \big)\trans \in \R^{d_0}$
as the homogeneous basis vector with the number of bases $d_0$, $\phi_{1}(x_h,a_h) = \big( \psi_{11}(x_{1h},a_h), \ldots, \psi_{1 d_1}(x_{1h},a_h) \big)\trans \in \R^{d_1}$ as the heterogeneous basis vector with the number of bases $d_1$ and $d = d_0 + d _1$. Examples of basis functions are B-splines or Fourier series basis terms. Then we can use the basis functions $\phi_0$ and  $\phi_1$ to estimate the rewards and value functions. In addition, we define a new multi-site confidence bound as 
\begin{equation}
   \overline \Gamma_{h}^k(x, a) =  \big( \alpha + \sqrt{N}(H+1) \eta\big)  \sqrt{\phi(x, a)\trans  (\Sigma_h^{k})^{-1}  \phi(x, a)}	+ (H +1) \eta\,,
\end{equation}
where $\alpha$ and $\eta$ are two tuning parameters.

\begin{theorem}[Suboptimality of FDTR under Nonparametric Estimation]
\label{theorem: spline}
With the tuning parameters $\alpha$ set as \eqref{tuning} and $\eta = c_3  m  d^{-(q+1)/m}$ for some constant $c_3$. Then we have, for any $k \in [K]$ and for any $x \in \Ssc$, $\widehat \pi^k = \{ \widehat \pi_h^k \}_{h=1}^H$ satisfies $\operatorname{SubOpt}^k\big(\widehat \pi^k ;x \big) \leq C H   N^{\frac{m}{ 2 (m + q + 1)}} \sqrt{ \zeta} \sum_{h=1}^H\E_{k,\pi^{k,*}}\Bigl[\sqrt{\phi(x_h,a_h)\trans ( \Sigma_h^k)^{-1}\phi(x_h,a_h) }\mid x_1=x\Bigr]  + C H  N^{\frac{-(q+1)}{2(m + q + 1)}}$ 
with probability at least $1- \xi$ for some constant $C$  dependent on $m$ and $q$.
\end{theorem}
Other basis functions can also be used such as multivariate trigonometric polynomials, with slightly different smoothness requirements for the transition kernel and rewards while yielding similar results. More discussion on the approximation property of these basis functions can be found in \cite{schultz1969multivariate}. Under the same well-explored assumption as Corollary \ref{cor:coverage}, we have the following corollary. 

\begin{corollary}
\label{cor:final}
Assume that there exists an absolute constant $b > 0$ such that $\Csc^k_h \ge b \bI_d$ for any $h \in [H]$ and $k \in [K]$ and $\| \phi_0(x,a)\|^2 \leq d_0 / d$ and $\|\phi_1(x, a)\|^2 \le d_1 / d$. If we choose the tuning parameters as Theorem \ref{theorem: spline}, it holds 
$\operatorname{SubOpt}^k\big(\widehat \pi^k ;x \big)  \le C H   N^{\frac{m}{ 2 (m + q + 1)}} \sqrt{ \zeta} \sqrt{ \frac{1}{b d } \Big( \frac{d_0}{N } +  \frac{d_1}{n_k} \Big) } +  C H  N^{\frac{-(q+1)}{2(m + q + 1)}}$
with probability at least $1- 2 \xi$.  
\end{corollary}

 When $d_1 = o(d/K)$, Corollary \ref{cor:final} reveals that the suboptimality of FDTR is $\tilde O( H  N^{\frac{-(q+1)}{2(m + q + 1)}}) $. When $q$ is large, indicating sufficiently smooth transition and reward functions, the term ${\frac{-(q+1)}{2(m + q + 1)}}$ approximate $-\frac{1}{2}$. This suggests that the rate from Corollary \ref{cor:final} converges to the optimal rate of $N^{-1/2}$.

\section{Experiments}
\label{sec:num}
\subsection{Simulations}
\label{sec: simulation}
We perform extensive simulations to evaluate the performance of FDTR. In particular, we focus on how well the proposed methods work as we vary the dimension and cardinality of the state and action space, respectively, episode length, and the number of sites. We simulate the data according to the following linear MDP. We generate random vectors $\theta_h^0\in\mathbb{R}^{d_0}$, $\theta_h^k\in\mathbb{R}^{d_1}$ for  $k\in[K]$ and matrix $\theta^\mu\in\mathbb{R}^{d_1\times H}$ following element-wise ${\rm Uniform}(0,1)$ distribution, which are then normalized to satisfy the assumptions in Section \ref{sec: theorem}. We generate rewards using a Gaussian distribution centered at $ \left\langle\phi_0(x, a), \theta_{h}^0\right\rangle + \left\langle\phi_1(x, a), \theta_{h}^k\right\rangle$ according to \eqref{def: reward}. Importance sampling is used to sample states from the state transitions density: given state action pair $(x,a)$ we draw the next state from a proposal distribution: $x'\sim q(\cdot|x,a)$ and re-weigh samples using $\P^k_{h}\left(x'|x,a\right)/q(\cdot|x,a)$ with $\P^k_{h}\left(x'| x, a\right)=\big\langle\phi_1(x, a), \mu_{h}^k (x')\big\rangle,$ and we use column $h$ in $\theta^\mu$ as $\mu_{h}^k$. 
We clip state vectors and normalize rewards such that $x_h^{k,\tau}\in\mathcal{S}$, $r_h^{k,\tau}\in[0,1]$ for all $h\in[H]$, $k\in[K]$, $\tau\in[n]$. Finally, we use linear functions with an action-interaction term to represent the treatment effect for $\phi_0$ and $\phi_1$, specifically, $\phi_l(x_h,a_h)=(x_{lh},a_h x_{lh})$ for $l=0,1$.

 Based on the theoretical analysis presented in Theorems \ref{theorem: prelim} and \ref{theorem: main}, we set the parameter $\lambda=1$. The parameter $\xi$ defines the probabilities for the suboptimality to be valid and is assigned a value nearing one, specifically $\xi=.99$. The remaining hyper-parameter $c$ within $\alpha$ is determined using $5$-fold cross-validation on the training data to optimize the estimated value functions across sites and trajectories, resulting in $c=.005$. We implement different benchmark methods for comparison. The first one is LDTR in Algorithm \ref{alg0}, which yields a policy $\widetilde \pi^{k} = \{\widetilde \pi_h^{k}  \}_{h=1}^H$ for a fixed $k$. We then use the $K$ locally trained LDTRs to define LDTR (MV), which given a state, uses majority voting across the $K$ policies to select an action. We also train a $Q$-learning policy in a single hospital site. We use ordinary least squares to estimate the $Q$-functions, as these are linear on $\phi(x,a)$. There are three variations of this method. The first, $Q$-learn $(1)$ uses a single $Q$-function for any time-step $h$ trained locally in each site. The second selects the most popular action among the $K$ locally trained $Q$ functions; we call this $Q$-learn (1-MV). The third one, $Q$-learn $(H)$, is also locally trained $Q$-learning; however, this one uses a different function $Q_h(\cdot,\cdot)$ for each of the $H$ time steps. 

\begin{figure}[htbp]
    \centering
    \begin{subfigure}[b]{0.31\textwidth}
        \includegraphics[width=\textwidth]{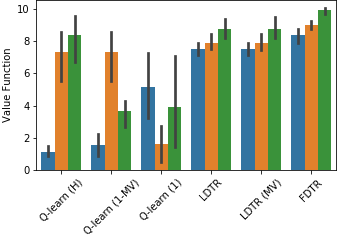}
        \caption{$(d,|\mathcal{A}|,H)=(8,6,15)$}
        \label{rfidtest_a}
    \end{subfigure}
    \begin{subfigure}[b]{0.3\textwidth}
        \includegraphics[width=\textwidth]{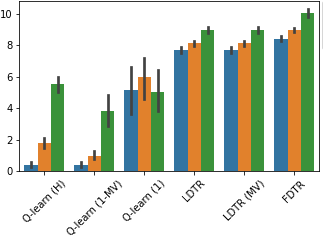}
        \caption{$(d,|\mathcal{A}|,H)=(8,6,15)$}
        \label{rfidtest_b}
    \end{subfigure}
    \begin{subfigure}[b]{0.37\textwidth}
        \includegraphics[width=\textwidth]{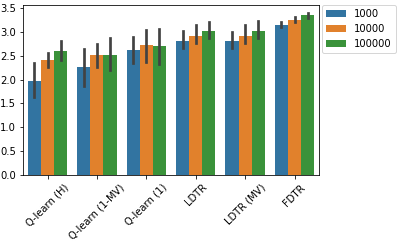}
        \caption{$(d,|\mathcal{A}|,H)=(20,2,5)$}
        \label{rfidtest_c}
    \end{subfigure}
    \caption{Mean value function for FDTR and benchmarks trained on $K=$5 sites for (\ref{rfidtest_a}), (\ref{rfidtest_c}) and $K=10$ for (\ref{rfidtest_b}). We show the value function averaged over the $K$ sites for increasing sample size. Error bars show $95\%$ CI. Finally, $(d,|\mathcal{A}|,H)$ stands respectively for the dimension of state space, cardinality of action space, and episode length.}
    \label{rfidtag_testing}
\end{figure}

Figure \ref{rfidtag_testing} shows empirical results for the performance of FDTR, LDTR, and its majority voting version, and their $Q$-learning counterparts for different settings. It is clear that across settings, FDTR outperforms in terms of the mean value. These results are expected since FDTR efficiently aggregates data across sites to estimate the policy. As the state dimension becomes large all methods decrease in performance for any given sample size, this is natural as estimated parameters have larger standard errors. In these cases such as Figure (\ref{rfidtest_c}), FDTR still performs best relative to local and majority voting methods. It is worth noting that the larger sample size benefits FDTR in terms of better policy estimation and also decreases uncertainty in terms of its performance, as illustrated by the narrower standard errors. Additionally, FDTR can estimate coefficients at the hospital level, which other methods cannot do locally. Hospital-level covariate estimation allows FDTR to tailor the policy to local hospital characteristics, translating into a better policy function.

\subsection{FDTR for Sepsis treatment Across Intensive Care Units}\label{sec: real}

We further illustrate the performance of FDTR in optimizing the treatment of sepsis, an acute and life-threatening condition in which the body reacts to infection improperly. Clinicians must determine fluid management strategies (actions) based on a patient's clinical state including vital signs and lab tests (state) at different stages (time) \citep{rhodes2017surviving}.  While survival is an important long-term reward to target, it is not a short-term parameter that can be tracked by physicians to adjust treatment strategies during hospitalization. We use an alternative reward based on serum lactate level which is an established marker for systemic tissue hypoperfusion that reflects cellular dysfunction in sepsis and is predictive of sepsis mortality \citep{lee2016new,ryoo2018lactate}. We identify optimal fluid management strategies for sepsis patients using longitudinal serum lactate levels as short-term rewards based on the MIMIC-IV EHR data \citep{mimiciV,raghu2017continuous}. 

Using publicly available SQL queries\footnote{\url{https://github.com/yugangjia/Team-Sepsis/blob/main/Sepsis\_cohort.sql}}, we extracted MIMIC-IV  sepsis cohort which consists of data from $9$ different care units with $25,568$ patient trajectories (episodes) \citep{sonabend2020expert}. The site-level covariates include an indicator for whether the care unit is not an ICU (noICU) and the patient flow  (Flow) defined as the total duration of all patients at each care unit (normalized by the total duration of all patients at all care units), which characterizes the size of the care unit. Patient-level covariates include measurements of weight, temperature, systolic blood pressure, hemoglobin, and potassium levels. We use the feature maps $\phi_l(x_h,a_h)=(x_{lh},a_h x_{lh},a_h^2 x_{lh})$ for $l=0,1$ to incorporate the second order action-interaction effect, yielding a state-space dimension of $21$. Each time step consists of a four-hour interval, and trajectories consist of $H=5$ time steps. The reward is proportionately inverse to the lactic acid level, which is a clinical index measuring the severity of Sepsis \citep{lee2016new}, { and we transform it to be between $0$ and $1$.}
The action space corresponds to the dosage median for intravenous fluids yielding three different actions. We use a step-importance sampling estimator \citep{thomas2016data, gottesman2018evaluating} for the value function, which we use to compare methods. We use a $50\%$-$50\%$ split at each care unit for training and test data, respectively. We train all methods described in Section \ref{sec: simulation} and evaluate them at each unit. The hyperparameters are chosen the same way as Section  \ref{sec: simulation}. 
Figure (\ref{a}) shows the value function estimate averaged over test sets. Estimates are computed for all methods along with the $95\%$ confidence interval. 
FDTR performs significantly better than the rest of the methods. This is because FDTR aggregates information across care units and can estimate site-specific effects, which yields a superior policy function.  As a sensitivity analysis, we have further fit a more complex model including non-linear bases and shown that including non-linear effects does not improve the value function, in part due to the bias-variance trade-off with limited sample size. The details and results are given in Supplementary S9.

\begin{figure}[ht]
    \centering
    \begin{subfigure}[b]{0.42\textwidth}
        \includegraphics[width=\textwidth]{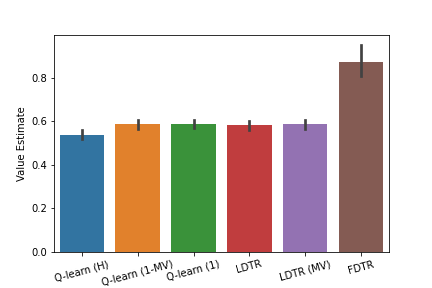}
        \caption{Estimated value function.}
        \label{a}
    \end{subfigure}
    \begin{subfigure}[b]{0.56\textwidth}
        \includegraphics[width=\textwidth]{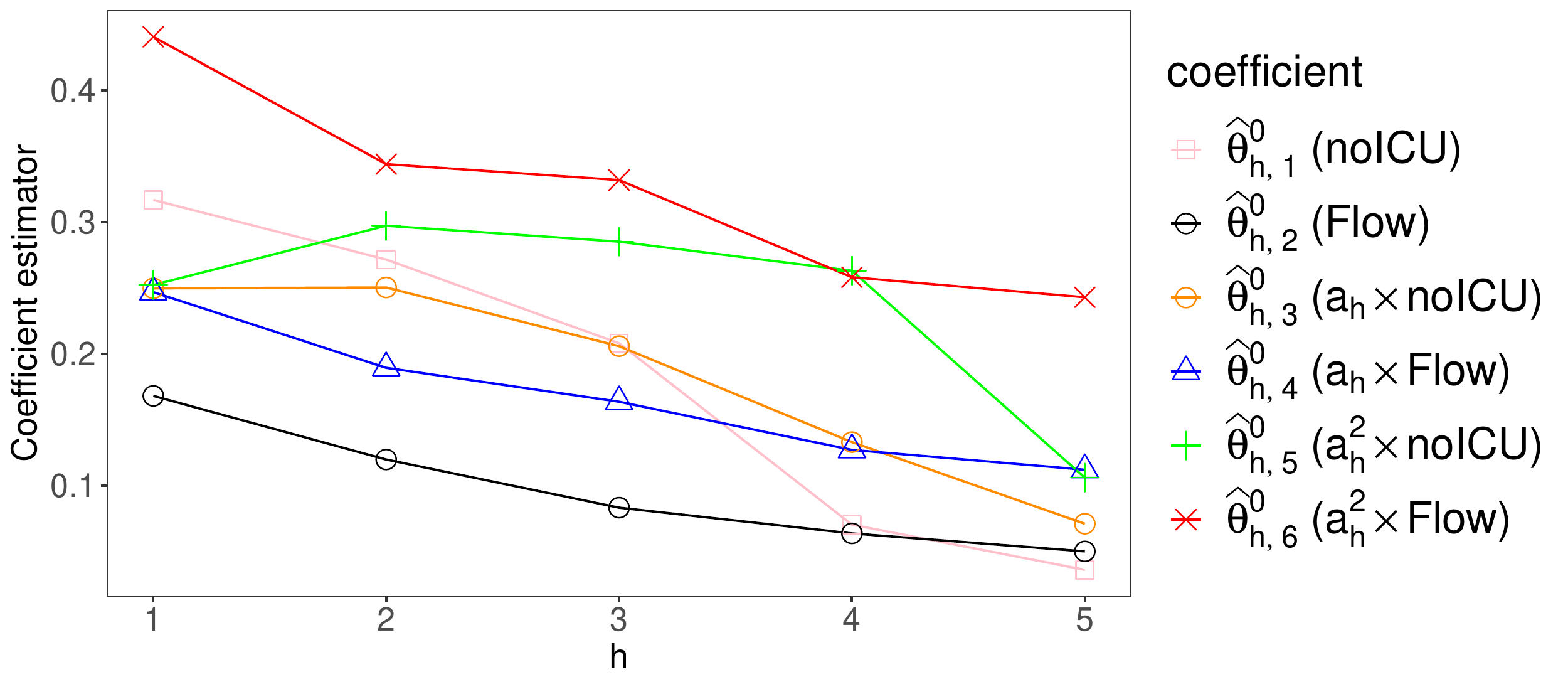}
        \caption{Estimated homogeneous coefficients.}
        \label{b}
    \end{subfigure}
     \caption{ (\ref{a}) Value function estimates on the sepsis 
     held out data across $9$ ICU. (\ref{b}) Homogeneous coefficients estimated by FDTR at different time points.}
     \label{fig3}
\end{figure}

 Finally, we show the homogeneous coefficients estimated by FDTR averaged over all care units in Figure (\ref{b}). First, we observe that all of the estimated coefficients are positive. For example, the positive coefficients related to the variable ``Flow" show that patients in larger care units (with a higher patient flow) might experience better outcomes, given the same treatment action. This could be due to factors such as the availability of resources, staffing levels, or specific expertise present in larger care units that may contribute to the more effective initial management of sepsis \citep{rudd2018global}. Furthermore, coefficients of site-level covariates tend to decrease over time. The trend implies that the influence of site-level covariates on treatment decisions decreases as the treatment progresses. This is expected because treatment decisions at later steps are more likely to be driven by patient-specific factors, such as their response to previous treatments and evolving clinical conditions, rather than site-level factors \citep{zhang2020individualized}.

In conclusion, the clinical interpretation of the FDTR results for sepsis management underscores the importance of understanding the influence of care unit characteristics and treatment actions on patient outcomes. This knowledge can help healthcare providers make more informed decisions and tailor their treatment strategies for better patient care.

\section{Discussion}
\label{sec: conclusion}

The current implementation of FDTR allows users to specify features with common effects across sites and features with site-specific effects. It is also plausible to allow for more data-adaptive co-training by imposing shrinkage estimation on the learned site-specific parameters, following strategies similar to the heterogeneity-aware federated regression methods \citep{duan2019heterogeneity}. We show the sample complexity of FDTR comparable to the homogeneous case under our model and it is interesting to find if the sample complexity involving shrinkage estimation will have similar results.

An interesting extension to our method would be to equip FDTR with doubly robust (DR) models. For example, if we additionally model the treatment propensity observed in the data, a DR version of FDTR could potentially achieve suboptimality even if the linear MDP assumption is incorrect. Another extension is to develop communication-efficient FDTR algorithms for more general models, which warrants future exploration.

\section*{Acknowledgement}

The authors thank the editor, the associate editor, and four referees for their constructive comments and suggestions. Tianxi Cai acknowledges the support of NIH (R01LM013614, R01HL089778). Junwei Lu acknowledges the support of NIH (R35CA220523, R01ES32418, U01CA209414). Zhaoran Wang acknowledges the support of NSF (Awards 2048075, 2008827, 2015568, 1934931), Simons Institute (Theory of Reinforcement Learning), Amazon, J.P. Morgan, and Two Sigma. The authors report there are no competing interests to declare. 






\bibliographystyle{chicago}
\bibliography{references}

\begin{thebibliography}{}

\bibitem[\protect\citeauthoryear{Agarwal, Kakade, Lee, and Mahajan}{Agarwal et~al.}{2021}]{agarwal2021theory}
Agarwal, A., S.~M. Kakade, J.~D. Lee, and G.~Mahajan (2021).
\newblock On the theory of policy gradient methods: Optimality, approximation, and distribution shift.
\newblock {\em The Journal of Machine Learning Research\/}~{\em 22\/}(1), 4431--4506.

\bibitem[\protect\citeauthoryear{Agarwal, Schuurmans, and Norouzi}{Agarwal et~al.}{2020}]{agarwal2020optimistic}
Agarwal, R., D.~Schuurmans, and M.~Norouzi (2020).
\newblock An optimistic perspective on offline reinforcement learning.
\newblock In {\em International Conference on Machine Learning}, pp.\  104--114. PMLR.

\bibitem[\protect\citeauthoryear{Agu, Pedersen, Strong, Tulu, He, Wang, and Li}{Agu et~al.}{2013}]{agu2013smartphone}
Agu, E., P.~Pedersen, D.~Strong, B.~Tulu, Q.~He, L.~Wang, and Y.~Li (2013).
\newblock The smartphone as a medical device: Assessing enablers, benefits and challenges.
\newblock In {\em 2013 IEEE International Workshop of Internet-of-Things Networking and Control (IoT-NC)}, pp.\  48--52. IEEE.

\bibitem[\protect\citeauthoryear{Antos, Szepesv\'{a}ri, and Munos}{Antos et~al.}{2008}]{NIPS2007_da0d1111}
Antos, A., C.~Szepesv\'{a}ri, and R.~Munos (2008).
\newblock Fitted q-iteration in continuous action-space mdps.
\newblock In J.~Platt, D.~Koller, Y.~Singer, and S.~Roweis (Eds.), {\em Advances in Neural Information Processing Systems}, Volume~20. Curran Associates, Inc.

\bibitem[\protect\citeauthoryear{Battey, Fan, Liu, Lu, and Zhu}{Battey et~al.}{2018}]{battey2018distributed}
Battey, H., J.~Fan, H.~Liu, J.~Lu, and Z.~Zhu (2018).
\newblock Distributed testing and estimation under sparse high dimensional models.
\newblock {\em Annals of statistics\/}~{\em 46\/}(3), 1352.

\bibitem[\protect\citeauthoryear{Cao, Zheng, Zhang, Yu, Piscitello, Zulueta, Ajilore, Ryan, and Leow}{Cao et~al.}{2017}]{cao2017deepmood}
Cao, B., L.~Zheng, C.~Zhang, P.~S. Yu, A.~Piscitello, J.~Zulueta, O.~Ajilore, K.~Ryan, and A.~D. Leow (2017).
\newblock Deepmood: modeling mobile phone typing dynamics for mood detection.
\newblock In {\em Proceedings of the 23rd ACM SIGKDD International Conference on Knowledge Discovery and Data Mining}, pp.\  747--755.

\bibitem[\protect\citeauthoryear{Chakraborty}{Chakraborty}{2013}]{chakraborty2013statistical}
Chakraborty, B. (2013).
\newblock {\em Statistical methods for dynamic treatment regimes}.
\newblock Springer.

\bibitem[\protect\citeauthoryear{Chakraborty and Murphy}{Chakraborty and Murphy}{2014}]{chakraborty2014dynamic}
Chakraborty, B. and S.~A. Murphy (2014).
\newblock Dynamic treatment regimes.
\newblock {\em Annual review of statistics and its application\/}~{\em 1}, 447--464.

\bibitem[\protect\citeauthoryear{Chen, Zhang, Giannakis, and Başar}{Chen et~al.}{2022}]{chen2018communication}
Chen, T., K.~Zhang, G.~B. Giannakis, and T.~Başar (2022).
\newblock Communication-efficient policy gradient methods for distributed reinforcement learning.
\newblock {\em IEEE Transactions on Control of Network Systems\/}~{\em 9\/}(2), 917--929.

\bibitem[\protect\citeauthoryear{Chen, Zhang, Zhang, Wang, and Zhu}{Chen et~al.}{2023}]{chen2022byzantine}
Chen, Y., X.~Zhang, K.~Zhang, M.~Wang, and X.~Zhu (2023).
\newblock Byzantine-robust online and offline distributed reinforcement learning.
\newblock In {\em International Conference on Artificial Intelligence and Statistics}, pp.\  3230--3269. PMLR.

\bibitem[\protect\citeauthoryear{Duan, Ning, and Chen}{Duan et~al.}{2022}]{duan2019heterogeneity}
Duan, R., Y.~Ning, and Y.~Chen (2022).
\newblock Heterogeneity-aware and communication-efficient distributed statistical inference.
\newblock {\em Biometrika\/}~{\em 109\/}(1), 67--83.

\bibitem[\protect\citeauthoryear{Duan, Jia, and Wang}{Duan et~al.}{2020}]{duan2020minimax}
Duan, Y., Z.~Jia, and M.~Wang (2020).
\newblock Minimax-optimal off-policy evaluation with linear function approximation.
\newblock In {\em International Conference on Machine Learning}, pp.\  2701--2709. PMLR.

\bibitem[\protect\citeauthoryear{Duchi, Jordan, and Wainwright}{Duchi et~al.}{2014}]{duchi2014privacy}
Duchi, J.~C., M.~I. Jordan, and M.~J. Wainwright (2014).
\newblock Privacy aware learning.
\newblock {\em Journal of the ACM (JACM)\/}~{\em 61\/}(6), 1--57.

\bibitem[\protect\citeauthoryear{Elgabli, Park, Bedi, Bennis, and Aggarwal}{Elgabli et~al.}{2020}]{elgabli2020gadmm}
Elgabli, A., J.~Park, A.~S. Bedi, M.~Bennis, and V.~Aggarwal (2020).
\newblock Gadmm: Fast and communication efficient framework for distributed machine learning.
\newblock {\em Journal of Machine Learning Research\/}~{\em 21\/}(76), 1--39.

\bibitem[\protect\citeauthoryear{Fan, Ma, Dai, Jing, Tan, and Low}{Fan et~al.}{2021}]{fan2021fault}
Fan, X., Y.~Ma, Z.~Dai, W.~Jing, C.~Tan, and B.~K.~H. Low (2021).
\newblock Fault-tolerant federated reinforcement learning with theoretical guarantee.
\newblock {\em Advances in Neural Information Processing Systems\/}~{\em 34}, 1007--1021.

\bibitem[\protect\citeauthoryear{Free, Phillips, Watson, Galli, Felix, Edwards, Patel, and Haines}{Free et~al.}{2013}]{free2013effectiveness}
Free, C., G.~Phillips, L.~Watson, L.~Galli, L.~Felix, P.~Edwards, V.~Patel, and A.~Haines (2013).
\newblock The effectiveness of mobile-health technologies to improve health care service delivery processes: a systematic review and meta-analysis.
\newblock {\em PLoS Med\/}~{\em 10\/}(1), e1001363.

\bibitem[\protect\citeauthoryear{Fujimoto, Meger, and Precup}{Fujimoto et~al.}{2019}]{fujimoto2019off}
Fujimoto, S., D.~Meger, and D.~Precup (2019).
\newblock Off-policy deep reinforcement learning without exploration.
\newblock In {\em International Conference on Machine Learning}, pp.\  2052--2062. PMLR.

\bibitem[\protect\citeauthoryear{Gottesman, Johansson, Komorowski, Faisal, Sontag, Doshi-Velez, and Celi}{Gottesman et~al.}{2019}]{gottesman2019guidelines}
Gottesman, O., F.~Johansson, M.~Komorowski, A.~Faisal, D.~Sontag, F.~Doshi-Velez, and L.~A. Celi (2019).
\newblock Guidelines for reinforcement learning in healthcare.
\newblock {\em Nature medicine\/}~{\em 25\/}(1), 16--18.

\bibitem[\protect\citeauthoryear{Gottesman, Johansson, Meier, Dent, Lee, Srinivasan, Zhang, Ding, Wihl, Peng, et~al.}{Gottesman et~al.}{2018}]{gottesman2018evaluating}
Gottesman, O., F.~Johansson, J.~Meier, J.~Dent, D.~Lee, S.~Srinivasan, L.~Zhang, Y.~Ding, D.~Wihl, X.~Peng, et~al. (2018).
\newblock Evaluating reinforcement learning algorithms in observational health settings.
\newblock {\em arXiv preprint arXiv:1805.12298\/}.

\bibitem[\protect\citeauthoryear{Gulcehre, Wang, Novikov, Paine, Colmenarejo, Zolna, Agarwal, Merel, Mankowitz, Paduraru, et~al.}{Gulcehre et~al.}{2020}]{gulcehre2020rl}
Gulcehre, C., Z.~Wang, A.~Novikov, T.~L. Paine, S.~G. Colmenarejo, K.~Zolna, R.~Agarwal, J.~Merel, D.~Mankowitz, C.~Paduraru, et~al. (2020).
\newblock Rl unplugged: Benchmarks for offline reinforcement learning.
\newblock {\em arXiv preprint arXiv:2006.13888\/}.

\bibitem[\protect\citeauthoryear{Hao, Li, Luo, Xu, Yang, and Liu}{Hao et~al.}{2019}]{hao2019efficient}
Hao, M., H.~Li, X.~Luo, G.~Xu, H.~Yang, and S.~Liu (2019).
\newblock Efficient and privacy-enhanced federated learning for industrial artificial intelligence.
\newblock {\em IEEE Transactions on Industrial Informatics\/}~{\em 16\/}(10), 6532--6542.

\bibitem[\protect\citeauthoryear{Hard, Rao, Mathews, Ramaswamy, Beaufays, Augenstein, Eichner, Kiddon, and Ramage}{Hard et~al.}{2018}]{hard2018federated}
Hard, A., K.~Rao, R.~Mathews, S.~Ramaswamy, F.~Beaufays, S.~Augenstein, H.~Eichner, C.~Kiddon, and D.~Ramage (2018).
\newblock Federated learning for mobile keyboard prediction.
\newblock {\em arXiv preprint arXiv:1811.03604\/}.

\bibitem[\protect\citeauthoryear{Harold, Kushner, and Yin}{Harold et~al.}{1997}]{harold1997stochastic}
Harold, J., G.~Kushner, and G.~Yin (1997).
\newblock Stochastic approximation and recursive algorithm and applications.
\newblock {\em Application of Mathematics\/}~{\em 35}.

\bibitem[\protect\citeauthoryear{Hastie, Tibshirani, and Friedman}{Hastie et~al.}{2009}]{hastie2009elements}
Hastie, T., R.~Tibshirani, and J.~H. Friedman (2009).
\newblock {\em The elements of statistical learning: data mining, inference, and prediction}, Volume~2.
\newblock Springer.

\bibitem[\protect\citeauthoryear{Hong, Rush, Liu, Zhou, Sun, Sonabend, Castro, Schubert, Panickan, Cai, et~al.}{Hong et~al.}{2021}]{hong2021clinical}
Hong, C., E.~Rush, M.~Liu, D.~Zhou, J.~Sun, A.~Sonabend, V.~M. Castro, P.~Schubert, V.~A. Panickan, T.~Cai, et~al. (2021).
\newblock {Clinical knowledge extraction via sparse embedding regression (KESER) with multi-center large scale electronic health record data}.
\newblock {\em NPJ digital medicine\/}~{\em 4\/}(1), 1--11.

\bibitem[\protect\citeauthoryear{Jadbabaie, Li, Qian, and Tian}{Jadbabaie et~al.}{2022}]{jadbabaie2022byzantine}
Jadbabaie, A., H.~Li, J.~Qian, and Y.~Tian (2022).
\newblock Byzantine-robust federated linear bandits.
\newblock In {\em 2022 IEEE 61st Conference on Decision and Control (CDC)}, pp.\  5206--5213. IEEE.

\bibitem[\protect\citeauthoryear{Jiang and Li}{Jiang and Li}{2016}]{jiang2016doubly}
Jiang, N. and L.~Li (2016).
\newblock Doubly robust off-policy value evaluation for reinforcement learning.
\newblock In {\em International Conference on Machine Learning}, pp.\  652--661. PMLR.

\bibitem[\protect\citeauthoryear{Jin, Yang, and Wang}{Jin et~al.}{2021}]{jin2020pessimism}
Jin, Y., Z.~Yang, and Z.~Wang (2021).
\newblock {Is pessimism provably efficient for offline RL?}
\newblock In {\em International Conference on Machine Learning}, pp.\  5084--5096. PMLR.

\bibitem[\protect\citeauthoryear{Johnson, Bulgarelli, Pollard, Horng, Celi, and Mark.}{Johnson et~al.}{2020}]{mimiciV}
Johnson, A., L.~Bulgarelli, T.~Pollard, S.~Horng, L.~A. Celi, and R.~Mark. (2020).
\newblock {MIMIC-IV (version 0.4). PhysioNet.}

\bibitem[\protect\citeauthoryear{Johnson, Pollard, Shen, Li-Wei, Feng, Ghassemi, Moody, Szolovits, Celi, and Mark}{Johnson et~al.}{2016}]{johnson2016mimic}
Johnson, A.~E., T.~J. Pollard, L.~Shen, H.~L. Li-Wei, M.~Feng, M.~Ghassemi, B.~Moody, P.~Szolovits, L.~A. Celi, and R.~G. Mark (2016).
\newblock {MIMIC-III}, a freely accessible critical care database.
\newblock {\em Scientific data\/}~{\em 3\/}(1), 1--9.

\bibitem[\protect\citeauthoryear{Jordan, Lee, and Yang}{Jordan et~al.}{2019}]{jordan2018communication}
Jordan, M.~I., J.~D. Lee, and Y.~Yang (2019).
\newblock Communication-efficient distributed statistical inference.
\newblock {\em Journal of the American Statistical Association\/}~{\em 114\/}(526), 668--681.

\bibitem[\protect\citeauthoryear{Kidambi, Rajeswaran, Netrapalli, and Joachims}{Kidambi et~al.}{2020}]{kumar2019stabilizing}
Kidambi, R., A.~Rajeswaran, P.~Netrapalli, and T.~Joachims (2020).
\newblock Morel: Model-based offline reinforcement learning.
\newblock {\em Advances in Neural Information Processing Systems\/}~{\em 33}, 21810--21823.

\bibitem[\protect\citeauthoryear{Kone{\v{c}}n{\`y}, McMahan, Yu, Richt{\'a}rik, Suresh, and Bacon}{Kone{\v{c}}n{\`y} et~al.}{2016}]{konevcny2016federated}
Kone{\v{c}}n{\`y}, J., H.~B. McMahan, F.~X. Yu, P.~Richt{\'a}rik, A.~T. Suresh, and D.~Bacon (2016).
\newblock Federated learning: Strategies for improving communication efficiency.
\newblock {\em arXiv preprint arXiv:1610.05492\/}.

\bibitem[\protect\citeauthoryear{Lange, Gabel, and Riedmiller}{Lange et~al.}{2012}]{lange2012batch}
Lange, S., T.~Gabel, and M.~Riedmiller (2012).
\newblock Batch reinforcement learning.
\newblock In {\em Reinforcement learning}, pp.\  45--73. Springer.

\bibitem[\protect\citeauthoryear{Lavori and Dawson}{Lavori and Dawson}{2004}]{lavori2004dynamic}
Lavori, P.~W. and R.~Dawson (2004).
\newblock Dynamic treatment regimes: practical design considerations.
\newblock {\em Clinical trials\/}~{\em 1\/}(1), 9--20.

\bibitem[\protect\citeauthoryear{Lee and An}{Lee and An}{2016}]{lee2016new}
Lee, S.~M. and W.~S. An (2016).
\newblock New clinical criteria for septic shock: serum lactate level as new emerging vital sign.
\newblock {\em Journal of thoracic disease\/}~{\em 8\/}(7), 1388.

\bibitem[\protect\citeauthoryear{Levine, Kumar, Tucker, and Fu}{Levine et~al.}{2020}]{levine2020offline}
Levine, S., A.~Kumar, G.~Tucker, and J.~Fu (2020).
\newblock Offline reinforcement learning: Tutorial, review, and perspectives on open problems.
\newblock {\em arXiv preprint arXiv:2005.01643\/}.

\bibitem[\protect\citeauthoryear{Li, Milletar{\`\i}, Xu, Rieke, Hancox, Zhu, Baust, Cheng, Ourselin, Cardoso, et~al.}{Li et~al.}{2019}]{li2019privacy}
Li, W., F.~Milletar{\`\i}, D.~Xu, N.~Rieke, J.~Hancox, W.~Zhu, M.~Baust, Y.~Cheng, S.~Ourselin, M.~J. Cardoso, et~al. (2019).
\newblock Privacy-preserving federated brain tumour segmentation.
\newblock In {\em International Workshop on Machine Learning in Medical Imaging}, pp.\  133--141. Springer.

\bibitem[\protect\citeauthoryear{Lim, Kim, Heo, and Han}{Lim et~al.}{2020}]{lim2020federated}
Lim, H.-K., J.-B. Kim, J.-S. Heo, and Y.-H. Han (2020).
\newblock Federated reinforcement learning for training control policies on multiple iot devices.
\newblock {\em Sensors\/}~{\em 20\/}(5), 1359.

\bibitem[\protect\citeauthoryear{Liu, Cai, Yang, and Wang}{Liu et~al.}{2019}]{liu2019neural}
Liu, B., Q.~Cai, Z.~Yang, and Z.~Wang (2019).
\newblock Neural trust region/proximal policy optimization attains globally optimal policy.
\newblock {\em Advances in Neural Information Processing Systems\/}~{\em 32}, 10565--10576.

\bibitem[\protect\citeauthoryear{Luckett, Laber, Kahkoska, Maahs, Mayer-Davis, and Kosorok}{Luckett et~al.}{2020}]{luckett2019estimating}
Luckett, D.~J., E.~B. Laber, A.~R. Kahkoska, D.~M. Maahs, E.~Mayer-Davis, and M.~R. Kosorok (2020).
\newblock Estimating dynamic treatment regimes in mobile health using v-learning.
\newblock {\em Journal of the American Statistical Association\/}~{\em 115\/}(530), 692--706.

\bibitem[\protect\citeauthoryear{McMahan, Moore, Ramage, Hampson, and y~Arcas}{McMahan et~al.}{2017}]{mcmahan2017communication}
McMahan, B., E.~Moore, D.~Ramage, S.~Hampson, and B.~A. y~Arcas (2017).
\newblock Communication-efficient learning of deep networks from decentralized data.
\newblock In {\em Artificial Intelligence and Statistics}, pp.\  1273--1282. PMLR.

\bibitem[\protect\citeauthoryear{Min, Yu, and Wang}{Min et~al.}{2019}]{min2019predictive}
Min, X., B.~Yu, and F.~Wang (2019).
\newblock Predictive modeling of the hospital readmission risk from patients’ claims data using machine learning: a case study on copd.
\newblock {\em Scientific reports\/}~{\em 9\/}(1), 1--10.

\bibitem[\protect\citeauthoryear{Murphy}{Murphy}{2003}]{murphy2003optimal}
Murphy, S.~A. (2003).
\newblock Optimal dynamic treatment regimes.
\newblock {\em Journal of the Royal Statistical Society: Series B (Statistical Methodology)\/}~{\em 65\/}(2), 331--355.

\bibitem[\protect\citeauthoryear{Murphy, van~der Laan, Robins, and Group}{Murphy et~al.}{2001}]{murphy2001marginal}
Murphy, S.~A., M.~J. van~der Laan, J.~M. Robins, and C.~P. P.~R. Group (2001).
\newblock Marginal mean models for dynamic regimes.
\newblock {\em Journal of the American Statistical Association\/}~{\em 96\/}(456), 1410--1423.

\bibitem[\protect\citeauthoryear{Nachum, Chow, Dai, and Li}{Nachum et~al.}{2019}]{nachum2019dualdice}
Nachum, O., Y.~Chow, B.~Dai, and L.~Li (2019).
\newblock Dualdice: Behavior-agnostic estimation of discounted stationary distribution corrections.
\newblock {\em Advances in Neural Information Processing Systems\/}~{\em 32}.

\bibitem[\protect\citeauthoryear{Nadiger, Kumar, and Abdelhak}{Nadiger et~al.}{2019}]{nadiger2019federated}
Nadiger, C., A.~Kumar, and S.~Abdelhak (2019).
\newblock Federated reinforcement learning for fast personalization.
\newblock In {\em 2019 IEEE Second International Conference on Artificial Intelligence and Knowledge Engineering (AIKE)}, pp.\  123--127. IEEE.

\bibitem[\protect\citeauthoryear{Parbhoo, Bogojeska, Zazzi, Roth, and Doshi-Velez}{Parbhoo et~al.}{2017}]{parbhoo2017combining}
Parbhoo, S., J.~Bogojeska, M.~Zazzi, V.~Roth, and F.~Doshi-Velez (2017).
\newblock Combining kernel and model based learning for hiv therapy selection.
\newblock {\em AMIA Summits on Translational Science Proceedings\/}~{\em 2017}, 239.

\bibitem[\protect\citeauthoryear{Puterman}{Puterman}{2014}]{puterman2014markov}
Puterman, M.~L. (2014).
\newblock {\em Markov decision processes: discrete stochastic dynamic programming}.
\newblock John Wiley \& Sons.

\bibitem[\protect\citeauthoryear{Raghu, Komorowski, Celi, Szolovits, and Ghassemi}{Raghu et~al.}{2017}]{raghu2017continuous}
Raghu, A., M.~Komorowski, L.~A. Celi, P.~Szolovits, and M.~Ghassemi (2017).
\newblock Continuous state-space models for optimal sepsis treatment: a deep reinforcement learning approach.
\newblock In {\em Machine Learning for Healthcare Conference}, pp.\  147--163. PMLR.

\bibitem[\protect\citeauthoryear{Rhodes, Evans, Alhazzani, Levy, Antonelli, Ferrer, Kumar, Sevransky, Sprung, Nunnally, et~al.}{Rhodes et~al.}{2017}]{rhodes2017surviving}
Rhodes, A., L.~E. Evans, W.~Alhazzani, M.~M. Levy, M.~Antonelli, R.~Ferrer, A.~Kumar, J.~E. Sevransky, C.~L. Sprung, M.~E. Nunnally, et~al. (2017).
\newblock Surviving sepsis campaign: international guidelines for management of sepsis and septic shock: 2016.
\newblock {\em Intensive care medicine\/}~{\em 43\/}(3), 304--377.

\bibitem[\protect\citeauthoryear{Robins}{Robins}{2004}]{robins2004optimal}
Robins, J.~M. (2004).
\newblock Optimal structural nested models for optimal sequential decisions.
\newblock In {\em Proceedings of the second seattle Symposium in Biostatistics}, pp.\  189--326. Springer.

\bibitem[\protect\citeauthoryear{Roomi, Shah, Ullah, Abedin, Butler, Schiers, Kohl, Yoo, Vibbert, and Jallo}{Roomi et~al.}{2021}]{roomi2021declining}
Roomi, S., S.~O. Shah, W.~Ullah, S.~U. Abedin, K.~Butler, K.~Schiers, B.~Kohl, E.~Yoo, M.~Vibbert, and J.~Jallo (2021).
\newblock Declining intensive care unit mortality of covid-19: a multi-center study.
\newblock {\em Journal of clinical medicine research\/}~{\em 13\/}(3), 184.

\bibitem[\protect\citeauthoryear{Rothchild, Panda, Ullah, Ivkin, Stoica, Braverman, Gonzalez, and Arora}{Rothchild et~al.}{2020}]{rothchild2020fetchsgd}
Rothchild, D., A.~Panda, E.~Ullah, N.~Ivkin, I.~Stoica, V.~Braverman, J.~Gonzalez, and R.~Arora (2020).
\newblock Fetchsgd: Communication-efficient federated learning with sketching.
\newblock In {\em International Conference on Machine Learning}, pp.\  8253--8265. PMLR.

\bibitem[\protect\citeauthoryear{Rudd, Kissoon, Limmathurotsakul, Bory, Mutahunga, Seymour, Angus, and West}{Rudd et~al.}{2018}]{rudd2018global}
Rudd, K.~E., N.~Kissoon, D.~Limmathurotsakul, S.~Bory, B.~Mutahunga, C.~W. Seymour, D.~C. Angus, and T.~E. West (2018).
\newblock The global burden of sepsis: barriers and potential solutions.
\newblock {\em Critical Care\/}~{\em 22\/}(1), 1--11.

\bibitem[\protect\citeauthoryear{Ryoo, Lee, Lee, Lee, Lim, Huh, Hong, Lim, Koh, and Kim}{Ryoo et~al.}{2018}]{ryoo2018lactate}
Ryoo, S.~M., J.~Lee, Y.-S. Lee, J.~H. Lee, K.~S. Lim, J.~W. Huh, S.-B. Hong, C.-M. Lim, Y.~Koh, and W.~Y. Kim (2018).
\newblock Lactate level versus lactate clearance for predicting mortality in patients with septic shock defined by sepsis-3.
\newblock {\em Critical care medicine\/}~{\em 46\/}(6), e489--e495.

\bibitem[\protect\citeauthoryear{Scherrer, Ghavamzadeh, Gabillon, Lesner, and Geist}{Scherrer et~al.}{2015}]{scherrer2015approximate}
Scherrer, B., M.~Ghavamzadeh, V.~Gabillon, B.~Lesner, and M.~Geist (2015).
\newblock Approximate modified policy iteration and its application to the game of tetris.
\newblock {\em Journal of Machine Learning Research\/}~{\em 16}, 1629--1676.

\bibitem[\protect\citeauthoryear{Schultz}{Schultz}{1969}]{schultz1969multivariate}
Schultz, M.~H. (1969).
\newblock L$^{\infty}$-multivariate approximation theory.
\newblock {\em SIAM Journal on Numerical Analysis\/}~{\em 6\/}(2), 161--183.

\bibitem[\protect\citeauthoryear{Sonabend, Lu, Celi, Cai, and Szolovits}{Sonabend et~al.}{2020}]{sonabend2020expert}
Sonabend, A., J.~Lu, L.~A. Celi, T.~Cai, and P.~Szolovits (2020).
\newblock Expert-supervised reinforcement learning for offline policy learning and evaluation.
\newblock In {\em Advances in Neural Information Processing Systems}, Volume~33, pp.\  18967--18977.

\bibitem[\protect\citeauthoryear{Sonabend-W, Laha, Ananthakrishnan, Cai, and Mukherjee}{Sonabend-W et~al.}{2023}]{sonabend2020semi}
Sonabend-W, A., N.~Laha, A.~N. Ananthakrishnan, T.~Cai, and R.~Mukherjee (2023).
\newblock Semi-supervised off-policy reinforcement learning and value estimation for dynamic treatment regimes.
\newblock {\em Journal of Machine Learning Research\/}~{\em 24\/}(323), 1--86.

\bibitem[\protect\citeauthoryear{Sutton and Barto}{Sutton and Barto}{2018}]{sutton2018reinforcement}
Sutton, R.~S. and A.~G. Barto (2018).
\newblock {\em Reinforcement learning: An introduction}.
\newblock MIT press.

\bibitem[\protect\citeauthoryear{Thomas and Brunskill}{Thomas and Brunskill}{2016}]{thomas2016data}
Thomas, P. and E.~Brunskill (2016).
\newblock Data-efficient off-policy policy evaluation for reinforcement learning.
\newblock In {\em International Conference on Machine Learning}, pp.\  2139--2148. PMLR.

\bibitem[\protect\citeauthoryear{Wang, Yang, Chen, and Liu}{Wang et~al.}{2019}]{wang2019distributed}
Wang, X., Z.~Yang, X.~Chen, and W.~Liu (2019).
\newblock Distributed inference for linear support vector machine.
\newblock {\em Journal of Machine Learning Research\/}~{\em 20}.

\bibitem[\protect\citeauthoryear{Xie, Cheng, Jiang, Mineiro, and Agarwal}{Xie et~al.}{2021}]{xie2021bellman}
Xie, T., C.-A. Cheng, N.~Jiang, P.~Mineiro, and A.~Agarwal (2021).
\newblock Bellman-consistent pessimism for offline reinforcement learning.
\newblock {\em Advances in neural information processing systems\/}~{\em 34}, 6683--6694.

\bibitem[\protect\citeauthoryear{Xie and Jiang}{Xie and Jiang}{2021}]{xie2021batch}
Xie, T. and N.~Jiang (2021).
\newblock Batch value-function approximation with only realizability.
\newblock In {\em International Conference on Machine Learning}, pp.\  11404--11413. PMLR.

\bibitem[\protect\citeauthoryear{Xie and Song}{Xie and Song}{2023}]{xie2022fedkl}
Xie, Z. and S.~Song (2023).
\newblock Fedkl: Tackling data heterogeneity in federated reinforcement learning by penalizing kl divergence.
\newblock {\em IEEE Journal on Selected Areas in Communications\/}~{\em 41\/}(4), 1227--1242.

\bibitem[\protect\citeauthoryear{Xu, Peng, Sun, Bhuiyan, Liu, and He}{Xu et~al.}{2021}]{xu2021fedmood}
Xu, X., H.~Peng, L.~Sun, M.~Z.~A. Bhuiyan, L.~Liu, and L.~He (2021).
\newblock Fedmood: Federated learning on mobile health data for mood detection.
\newblock {\em arXiv preprint arXiv:2102.09342\/}.

\bibitem[\protect\citeauthoryear{Yang, Liu, Chen, and Tong}{Yang et~al.}{2019}]{yang2019federated}
Yang, Q., Y.~Liu, T.~Chen, and Y.~Tong (2019).
\newblock Federated machine learning: Concept and applications.
\newblock {\em ACM Transactions on Intelligent Systems and Technology (TIST)\/}~{\em 10\/}(2), 1--19.

\bibitem[\protect\citeauthoryear{Zhan, Huang, Huang, Jiang, and Lee}{Zhan et~al.}{2022}]{zhan2022offline}
Zhan, W., B.~Huang, A.~Huang, N.~Jiang, and J.~Lee (2022).
\newblock Offline reinforcement learning with realizability and single-policy concentrability.
\newblock In {\em Conference on Learning Theory}, pp.\  2730--2775. PMLR.

\bibitem[\protect\citeauthoryear{Zhang, Yang, and Ba{\c{s}}ar}{Zhang et~al.}{2019}]{zhang2019multi}
Zhang, K., Z.~Yang, and T.~Ba{\c{s}}ar (2019).
\newblock Multi-agent reinforcement learning: A selective overview of theories and algorithms.
\newblock {\em arXiv preprint arXiv:1911.10635\/}.

\bibitem[\protect\citeauthoryear{Zhang, Dai, Li, and Schuurmans}{Zhang et~al.}{2020}]{zhang2020gendice}
Zhang, R., B.~Dai, L.~Li, and D.~Schuurmans (2020).
\newblock Gendice: Generalized offline estimation of stationary values.
\newblock {\em arXiv preprint arXiv:2002.09072\/}.

\bibitem[\protect\citeauthoryear{Zhang, Zheng, and Liu}{Zhang et~al.}{2020}]{zhang2020individualized}
Zhang, Z., B.~Zheng, and N.~Liu (2020).
\newblock Individualized fluid administration for critically ill patients with sepsis with an interpretable dynamic treatment regimen model.
\newblock {\em Scientific Reports\/}~{\em 10\/}(1), 1--9.

\bibitem[\protect\citeauthoryear{Zhuo, Feng, Lin, Xu, and Yang}{Zhuo et~al.}{2019}]{zhuo2019federated}
Zhuo, H.~H., W.~Feng, Y.~Lin, Q.~Xu, and Q.~Yang (2019).
\newblock Federated deep reinforcement learning.
\newblock {\em arXiv preprint arXiv:1901.08277\/}.

\end{thebibliography}

\end{document}


\begin{center}
\textit{\large Supplementary Material to}
\end{center}

\begin{center}
{\LARGE Federated Offline Reinforcement Learning}
\vskip10pt
\end{center}

\setcounter{section}{0}
\renewcommand{\thesection}{S.\arabic{section}}
\setcounter{equation}{0}
\counterwithout{equation}{section}
\renewcommand{\theequation}{S.\arabic{equation}}
\setcounter{theorem}{0}
\counterwithout{theorem}{section}
\renewcommand{\thetheorem}{S.\arabic{theorem}}
\renewcommand{\thelemma}{S.\arabic{theorem}}

\renewcommand{\thesection}{S\arabic{section}}  
\renewcommand{\thetable}{S\arabic{table}}  
\renewcommand{\thefigure}{S\arabic{figure}}
\renewcommand{\theequation}{S\arabic{equation}}

\section{Proof of Theorem 1}

Algorithm 1 is the adoption of Algorithm $2$ of \cite{jin2020pessimism} to the linear MDP model, and thus Theorem 1 follows straightforwardly from Theorem $4.4$ of \cite{jin2020pessimism}.

\section{Proof of Theorem 2}

For the $k$th site and any function $f: \mathcal{S} \rightarrow \mathbb{R}$, we define the
transition operator as
$$
(\P_{h}^k f)(x, a)=\E_k [f (x_{h+1}) \mid  x_{h}=x, a_{h}=a]\,.
$$
By (5), we know that there exists $\theta_h = \big( (\theta_h^0)\trans, (\bar \theta_h^1)\trans,\ldots, (\bar \theta_h^K)\trans \big)\trans \in \R^{d_K}$ such that $\B^k_{h} \widehat{V}^k_{h+1} = \phi^k(x, a) \trans  \theta_h $ and $\B^j_{h} \widetilde{V}^j_{h+1} = \phi^j(x, a) \trans  \theta_h$, for $j \in [K]/\{k\}$. To simplify the notation, denote $Y_h \in \R^{N}$ as the concatenation of $Y_h^k(\widehat V_{h+1}^k)$ and $Y_h^j(\widetilde V_{h+1}^j)$ for $j \in [K]/\{k\}$. Now we upper bound the difference between $\B^k_{h} \widehat{V}^k_{h+1}$ and $\widehat{\B}^k_{h} \widehat{V}^k_{h+1}$. For all $h \in[H]$ and all $(x, a) \in$
$\mathcal{S} \times \Asc$, we have
$$
\begin{array}{l}
(\B^k_{h} \widehat{V}^k_{h+1})(x, a)-(\widehat{\B}^k_{h} \widehat{V}^k_{h+1})(x, a)= \phi^k(x, a) \trans ( \theta_h - \widehat \theta_{h,k}) \\
= \phi^k(x, a) \trans  \theta_h  -  \phi^k(x, a) \trans(\Lambda_h  +\bH_{k,\lambda})^{+} \Phi_h\trans  Y_h \\
=\underbrace{\phi^k(x, a)\trans  \theta_h - \phi^k(x, a)\trans  (\Lambda_h  +\bH_{k,\lambda})^{+} \Phi_h\trans  \Phi_h   \theta_h}_{(\mathrm{i})} -\underbrace{\phi^k(x, a)\trans (\Lambda_h  +\bH_{k,\lambda})^{+}  \Phi_h\trans ( Y_h -  \Phi_h \theta_h )}_{(\mathrm{ii})} .
\end{array}
$$
Thus, by triangle inequality, we have
$$
\big|(\B_{h}^k \widehat{V}^k_{h+1})(x, a)-(\widehat{\B}_{h}^k \widehat{V}^k_{h+1})(x, a)\big| \leq|(\mathrm{i})|+|(\mathrm{ii})| .
$$
In the sequel, we bound term (i) and term (ii) separately. Since $\widehat{V}^k_{h+1} \in[0, H-1]$, by Lemma \ref{lemma: norm}, we have 
$\|\theta_h^0\|^2 +  \|\bar \theta_h^k\|^2\leq H^2 d$. By direct computation, we bound term (i) by 
\begin{equation*}
\begin{aligned}
 |(\mathrm{i})| & =  |\phi^k(x, a)\trans \big(  \theta_h - (\Lambda_h  +\bH_{k,\lambda})^{+}(\Lambda_h +\bH_{k,\lambda} - \bH_{k,\lambda})\theta_h \big)| \\
 & =  |\phi^k(x, a)\trans  (\Lambda_h  +\bH_{k,\lambda})^{+} \bH_{k,\lambda} \theta_h|  =  \lambda  \phi(x, a)\trans (\Sigma_h^{k})^{-1}  \big( (\theta_h^0)\trans, (\theta_h^k)\trans \big)\trans\\
 & \leq   \lambda  \|\big( (\theta_h^0)\trans, (\theta_h^k)\trans \big)\trans\|_{(\Sigma_h^{k})^{-1}  } \|\phi(x, a)\|_{(\Sigma_h^{k})^{-1} } \leq H \sqrt{ d \lambda} \sqrt{\phi(x, a)\trans(\Sigma_h^{k})^{-1} \phi(x, a)} \,.
\end{aligned}
\end{equation*}
Here the third equality comes from the matrix block-wise inverse formula, the first inequality comes from Cauchy-Schwarz inequality, and the last inequality follows from the fact that $$
\big\|\big( (\theta_h^0)\trans, (\theta_h^k)\trans \big)\trans\big\|_{(\Sigma_h^{k})^{-1}}\leq\big\|(\Sigma_h^{k})^{-1}\big\|_{\rm op}^{1 / 2}  \big\|\big( (\theta_h^0)\trans, (\theta_h^k)\trans \big)\trans\big\| \leq H \sqrt{ d / \lambda}.
$$
It remains to bound term (ii). To simplify the notation, for any
$j \in [K], h \in[H]$ and $\tau \in[n_j]$, and any value function $V: \mathcal{S} \rightarrow[0, H]$, we define
\begin{equation*}
   \epsilon_{h}^{j,\tau}(V) = r_{h}^{j,\tau}+V(x_{h+1}^{j,\tau})-(\B^j_{h} V)(x_{h}^{j,\tau}, a_{h}^{j,\tau}); \; \epsilon_{h}^{j}(V) = (\epsilon_{h}^{j,1}(V),\ldots, \epsilon_{h}^{j,n_j}(V))\trans
\end{equation*}
and 
$
\epsilon_{h} = \big( (\epsilon_{h}^1)\trans, \ldots, (\epsilon_{h}^K)\trans \big)\trans$, where $\epsilon_{h}^j = \big( \epsilon_{h}^{j,1}(\widetilde V^j_{h+1}) , \ldots, \epsilon_{h}^{j,n_j}(\widetilde V^j_{h+1}) \big)\trans
$
for $j \in [K]/\{k\}$, and $\epsilon_{h}^k = \big( \epsilon_{h}^{j,1}(\widehat V^k_{h+1}) , \ldots, \epsilon_{h}^{j,n_k}(\widehat V^k_{h+1}) \big)\trans$. Using such a notation, we bound term (ii) using Cauchy-Schwarz inequality as
\begin{equation*}
\begin{aligned}
 |(\mathrm{ii})| & =\left|\phi^k(x, a)\trans (\Lambda_h  +\bH_{k,\lambda})^{+}  \Phi_h\trans \epsilon_h \right| \\
& = \left| \phi(x, a)\trans  (\Sigma_h^{k})^{-1} \begin{bmatrix}
(\Phi_{0,h}^k)\trans \epsilon_{h}^k - \sum_{j \neq k}^{K}  (\Phi_{0,h}^j)\trans ( \bI_{n_j} - \bP_{\Phi_{1,h}^j} ) \epsilon_{h}^j \\
(\Phi_{1,h}^k)\trans \epsilon_{h}^k 
\end{bmatrix} \right | \\
& \leq \underbrace{\left\|  \begin{bmatrix}
(\Phi_{0,h}^k)\trans \epsilon_{h}^k - \sum_{j \neq k}^{K}  (\Phi_{0,h}^j)\trans \big( \bI_{n_j} - \bP_{\Phi_{1,h}^j} \big) \epsilon_{h}^j \\
(\Phi_{1,h}^k)\trans \epsilon_{h}^k 
\end{bmatrix}  \right\|_{(\Sigma_h^{k})^{-1}}}_{(\mathrm{iii})} \sqrt{\phi(x, a)\trans (\Sigma_h^{k})^{-1} \phi(x, a)} \,.
\end{aligned}
\end{equation*}
In the sequel, we upper bound term (iii) via concentration inequalities. An obstacle is that $\widehat{V}^k_{h+1}$ and $\widetilde{V}^j_{h+1}, j \neq k$ depends on $\big\{(x_{h}^{j,\tau}, a_{h}^{j,\tau})\big\}_{j,\tau=1}^{K,n_j}$ via $\big\{(x_{h^{\prime}}^{j,\tau}, a_{h^{\prime}}^{j,\tau})\big\}_{j \in [K],\tau \in[n_j], h^{\prime}>h}$, as it is constructed based on the dataset $\Dsc$. To this end, we resort to uniform concentration inequalities to upper bound
\begin{equation}
    \sup _{V^j \in \mathcal{V}_{h+1}(R, B, \lambda), j \in [K]} \underbrace{
\left\|  \begin{bmatrix}
(\Phi_{0,h}^k)\trans \epsilon_{h}^k(V^k) - \sum_{j \neq k}^{K}  (\Phi_{0,h}^j)\trans ( \bI_{n_j} - \bP_{\Phi_{1,h}^j} ) \epsilon_{h}^j(V^j) \\
(\Phi_{1,h}^k)\trans \epsilon_{h}^k (V^k)
\end{bmatrix}  \right\|_{(\Sigma_h^{k})^{-1}}}_{f\big( \{ V^{j}\}_{j \in [K]}\big)}
\label{def: f}
\end{equation}
for each $h \in[H]$, where it holds that $\widehat{V}^k_{h+1}, \widetilde{V}^j_{h+1}, j \neq k \in \mathcal{V}_{h+1}(R, B, \lambda)$. Here for all $h \in[H]$, we define the
function class
\begin{equation}
    \mathcal{V}_{h}(R, B, \lambda)=\left\{V_{h}(x ; \theta, \alpha, \Sigma): \mathcal{S} \rightarrow[0, H]\text{ with }  \|\theta\| \leq R, \alpha \in[0, B], \Sigma \succeq \lambda  \bI_d \right\}
\end{equation}
where $V_{h}(x ; \theta, \alpha, \Sigma)=\max _{a \in \Asc}\big\{\min \big\{\phi(x, a)\trans \theta-\alpha  \sqrt{\phi(x, a)\trans \Sigma^{-1} \phi(x, a)}, H-h+1\big\}^{+}\big\}.$ Notice that we use the feature map $\phi \in \R^d$ rather than $\phi^k \in \R^{d_K}$. 

For all $\epsilon>0$ and all $h \in[H]$, let $\mathcal{N}_{h}(\epsilon ; R, B, \lambda)$ be the minimal $\epsilon$-cover of $\mathcal{V}_{h}(R, B, \lambda)$ with respect to the supremum norm. In other words, for any function $V \in \mathcal{V}_{h}(R, B, \lambda)$, there exists a function $V^{\dagger} \in \mathcal{N}_{h}(\epsilon ; R, B, \lambda)$ such that
$$
\sup _{x \in \mathcal{S}}\big|V(x)-V^{\dagger}(x)\big| \leq \epsilon \,.
$$
Meanwhile, among all $\epsilon$-covers of $\mathcal{V}_{h}(R, B, \lambda)$ defined by such a property, we choose $\mathcal{N}_{h}(\epsilon ; R, B, \lambda)$
as the one with the minimal cardinality. 

Recall the construction of $\widetilde V_h^j$ in the Algorithm 1, and $\alpha_j \leq \alpha$ for $j \in [K]$ and $j \neq k$,  By Lemma \ref{lemma: norm}, we have
$\bigl \|\big( (\widehat{\theta}_{h,k}^0)\trans,(\widehat{\theta}_{h,k}^k)\trans \big)\trans  \bigr \| \leq H \sqrt{N / \lambda}$ and $\|\widetilde \theta_h^j\| \leq H \sqrt{n_j / \lambda}$. Hence, for all $h \in[H]$, we have
$$
\widehat{V}_{h+1}^k, \widetilde{V}_{h+1}^j \in \mathcal{V}_{h+1}\left(R_{0}, B_{0}, \lambda\right), \quad \text { where } R_{0}=H \sqrt{N / \lambda}, \text{ and } B_{0}=2 \alpha \,.
$$
Here $\lambda>0$ is the regularization parameter and $\alpha>0$ is the scaling parameter, which are specified
in Algorithm 2. For notational simplicity, we use $\mathcal{V}_{h+1}$ and $\mathcal{N}_{h+1}(\epsilon)$ to denote $\mathcal{V}_{h+1}(R_{0}, B_{0}, \lambda)$ and $\mathcal{N}_{h+1}\left(\epsilon ; R_{0}, B_{0}, \lambda\right)$, respectively. As it holds that $\widehat{V}^k_{h+1}, \widetilde{V}^j_{h+1}, j \neq k \in \mathcal{V}_{h+1}$ and $\mathcal{N}_{h+1}(\epsilon)$ is an $\epsilon$-cover of $\mathcal{V}_{h+1}$, there exists functions $V_{h+1}^{j,\dagger} \in \mathcal{N}_{h+1}(\epsilon) , j\in [K]$ such that
$$
\sup _{x \in \mathcal{S}} |\widehat{V}^k_{h+1}(x)-V_{h+1}^{k, \dagger}(x)| \leq \epsilon; \; \sup _{x \in \mathcal{S}}|\widetilde{V}^j_{h+1}(x)-V_{h+1}^{j, \dagger}(x)| \leq \epsilon , j \in [K]/\{k\}. 
$$
Hence, given $V_{h+1}^{j, \dagger}, j \in [K]$ and $\widehat{V}^k_{h+1}$, $\widetilde{V}^j_{h+1}, j \in [K]/\{k\}$ , the monotonicity of conditional expectations implies
\begin{equation}
    \begin{array}{l}
\big|(\P_{h}^k V_{h+1}^{k,\dagger})(x, a)- (\P^k_{h} \widehat{V}^k_{h+1})(x, a)\big| \\
\quad=\big|\E[V_{h+1}^{k,\dagger}(x_{h+1})  - \widehat{V}^k_{h+1} (x_{h+1}) \mid  x_{h}=x, a_{h}=a] \big| \\
\quad \leq \E\big[ |V_{h+1}^{k,\dagger}(x_{h+1}) - \widehat{V}^k_{h+1} (x_{h+1})| \mid  x_{h}=x, a_{h}=a \big] \\ \quad \leq \epsilon, \quad \forall(x, a) \in \mathcal{S} \times \Asc, \forall h \in[H]
\end{array}
\label{eq: PhV}
\end{equation}
Here the conditional expectation is induced by the transition kernel $\P^k_{h}(\cdot \mid x, a)$. Similarly we have
$$\big|(\P_{h}^j V_{h+1}^{j,\dagger})(x, a)-(\P^j_{h} \widetilde{V}^j_{h+1})(x, a)\big| \leq \epsilon,  j \in [K]/\{k\}.$$
Combining \eqref{eq: PhV} and the definition of the Bellman operator $\B^j_{h}$ in (5), we have
$$
\big|(\B^k_{h} V_{h+1}^{k,\dagger})(x, a)-(\B^k_{h} \widehat{V}^k_{h+1})(x, a)\big| \leq \epsilon, \quad \forall(x, a) \in \mathcal{S} \times \Asc, \forall h \in[H]
$$
and similarly
$$
\big|(\B^j_{h} V_{h+1}^{j,\dagger})(x, a)-(\B^j_{h} \widetilde{V}^j_{h+1})(x, a)\big| \leq \epsilon, \quad \forall(x, a) \in \mathcal{S} \times \Asc, \forall h \in[H], j \in [K]/\{k\}. 
$$
As a result, we have
$$
\big|\epsilon_{h}^{k,\tau}(\widehat{V}^k_{h+1})-\epsilon_{h}^{k,\tau}(V_{h+1}^{k,\dagger})\big| \leq 2 \epsilon,  \forall \tau \in[n_k]; \; \big|\epsilon_{h}^{j,\tau}(\widetilde{V}^k_{h+1})-\epsilon_{h}^{j,\tau}(V_{h+1}^{j,\dagger})\big| \leq 2 \epsilon, \forall j \in [K]/\{k\}, \forall \tau \in[n_j]
$$
for all $h \in[H]$. 

Also, recall the definition of term (iii). By the Cauchy-Schwarz inequality, for any two vectors $a, b \in \mathbb{R}^{d}$ and any positive definite matrix $\Lambda \in \mathbb{R}^{d \times d}$, it holds that $\|a+b\|_{\Lambda}^{2} \leq$ $2 \|a\|_{\Lambda}^{2}+2 \|b\|_{\Lambda}^{2} .$ Hence, for all $h \in[H]$, we have
\begin{equation}
\begin{aligned}
     &  |(\mathrm{iii})|^{2} \leq 2  \underbrace{  \left\|  \begin{bmatrix}
(\Phi_{0,h}^k)\trans \epsilon_{h}^k(V^{k,\dagger}) - \sum_{j \neq k}^{K}  (\Phi_{0,h}^j)\trans \big( \bI_{n_j} - \bP_{\Phi_{1,h}^j}  \big) \epsilon_{h}^j(V^{j,\dagger}) \\
(\Phi_{1,h}^k)\trans \epsilon_{h}^k (V^{k,\dagger})
\end{bmatrix}  \right\|_{(\Sigma_h^{k})^{-1}}^2 }_{ f( \{ V^{j,\dagger}\}_{j \in [K]})}\\
    & + \underbrace{ 2  \left\|  \begin{bmatrix}
(\Phi_{0,h}^k)\trans (\epsilon_{h}^k(V^{k,\dagger}) - \epsilon_{h}^k) - \sum_{j \neq k}^{K}  (\Phi_{0,h}^j)\trans \big( \bI_{n_j} -\bP_{\Phi_{1,h}^j}  \big) (\epsilon_{h}^j(V^{j,\dagger}) -  \epsilon_{h}^j)\\
(\Phi_{1,h}^k)\trans  (\epsilon_{h}^k(V^{k,\dagger}) - \epsilon_{h}^k)
\end{bmatrix}  \right\|_{(\Sigma_h^{k})^{-1}}^2}_{ A }
\end{aligned}
\label{eq: iii}   
\end{equation}
The second term $A$ on the right-hand side of \eqref{eq: iii} is upper bounded by
\begin{equation}
\begin{aligned}
&A \leq  2 \left\| (\Sigma_{h}^k)^{-1/2} \begin{bmatrix}
 (\Phi_{0,h}^1)\trans ( \bI_{n_1} - \bP_{\Phi_{1,h}^j} ) & \ldots & (\Phi_{0,h}^k)\trans & \ldots \\
 \mathbf{0} & \ldots & (\Phi_{1,h}^k)\trans & \ldots 
 \end{bmatrix} \right\|_{\rm op}^{2} \sum_{j=1}^{K} \left\|\epsilon_{h}^k(V^{j,\dagger}) - \epsilon_{h}^j\right\|^{2} \\
 & \leq 8N\epsilon^2.
\end{aligned}    
\label{eq: iii 2}
\end{equation}
We then bound $f( \{ V^{j,\dagger}\}_{j \in [K]})$ via uniform concentration inequalities.  Applying Lemma \ref{lemma: Concentration} and the union bound, for any fixed $h \in [H]$, we have 
$$
\begin{aligned}
 & \P_{\Dsc} \bigg(  \sup _{V^j \in \mathcal{N}_{h+1}(\epsilon), j \in [K] } f( \{ V^{j}\}_{j \in [K]}) \leq H^2  \Bigl ( 2 \log \bigl (1/\delta \bigr )   + d \log (1 + N / \lambda )  \Bigr ) \bigg)\\ 
 & \leq   \delta  | \mathcal{N}_{h+1}(\epsilon ) |.
\end{aligned}
$$
For all $\xi \in(0, 1)$ and all $\epsilon > 0$, we set  $\delta =\xi/(H |\mathcal{N}_{h+1}(\epsilon)| )$.  
Hence, for any fixed $h \in [H]$, it holds that
\begin{equation}
\begin{aligned}
& \sup _{V^j \in \mathcal{N}_{h+1}(\epsilon), j \in [K] } f( \{ V^{j}\}_{j \in [K]})  \\
& \leq H^2  \Bigl ( 2 \log \bigl (H   | \mathcal{N}_{h+1} (\epsilon ) | / \xi  \bigr )   + d  \log (1 + N / \lambda )  \Bigr )    
\end{aligned}
   \label{eq:bound_term3_4}
\end{equation}
with probability at least $1 - \xi / H$, which is taken with respect to $\P_{\Dsc}$. Using the union bound again, we have \eqref{eq:bound_term3_4} holds for all $h \in [H]$ with probability at least $1 - \xi$. 

Combining \eqref{eq: iii 2} and \eqref{eq:bound_term3_4}, we have 

$$ |(\mathrm{iii})|^{2} \leq   2 H^2 \Bigl ( 2  \log \bigl (H   | \mathcal{N}_{h+1} (\epsilon ) | / \xi  \bigr )   + d  \log (1 + N / \lambda )  \Bigr )  + 8  N \epsilon^2.$$
Recall that
$$
\widehat{V}^k_{h+1}, \widetilde{V}^j_{h+1} \in \mathcal{V}_{h+1}\left(R_{0}, B_{0}, \lambda\right),  j \in [K]/\{k\},
$$
where $R_{0}=H \sqrt{N / \lambda},  B_{0}=2 \alpha, \lambda=1, \alpha = c  d H \sqrt{\zeta}$. Here $c>0$ is an absolute constant, $\xi \in(0,1)$ is the confidence parameter, and $\zeta=\log (2 d H N / \xi)$ is specified in Algorithm 2. Recall that $\mathcal{N}_{h+1}(\epsilon)=\mathcal{N}_{h+1}\left(\epsilon ; R_{0}, B_{0}, \lambda\right)$ is the minimal $\epsilon$-cover of $\mathcal{V}_{h+1}=\mathcal{V}_{h+1}\left(R_{0}, B_{0}, \lambda\right)$ with respect to the supremum norm. Applying Lemma \ref{lemma: covering number} with $\epsilon = H \sqrt{d }/\sqrt N$, $\alpha = c d H \sqrt{\zeta}$, and $\lambda = 1$,
we have
$$
\begin{aligned}
\log \left|\mathcal{N}_{h}(\epsilon ; R_0, B_0, \lambda)\right| & \leq d  \log (1+4 R_0 / \epsilon)+d^{2}  \log \left(1+8 d^{1 / 2} B_0^{2} /\left(\epsilon^{2}
\lambda\right)\right) \\
& = d  \log (1+4 H \sqrt{N / \lambda}/ \epsilon)+d^{2}  \log \left(1+8 d^{1 / 2} 4 \beta^{2} /\left(\epsilon^{2}
\lambda\right)\right) \\
 & \leq d  \log (1+4 N/ \sqrt{\lambda d} )+d^{2}  \log \left(1+ 32 d^{1 / 2}  \beta^{2} N /\left(d H^2
\lambda\right)\right) \\
& =  d  \log (1+4 N/ \sqrt{\lambda d} )+d^{2}  \log (1+ 32 d^{3 / 2}  c^2 \zeta N /\lambda) \\
& = d  \log (1+4 N/ \sqrt{d} )+d^{2}  \log (1+ 32 d^{3 / 2}  c^2 \zeta N) \\
& \leq 2 d^{2}  \log (1+ 32 d^{3 / 2}  c^2 \zeta N) \leq 2 d^{2}  \log (64 c^2  d^{3 / 2}  N \zeta ),
\end{aligned}
$$
which implies that
$$
\begin{aligned}
 &  2 H^2  \Bigl ( 2  \log \bigl (H   | \mathcal{N}_{h+1} (\epsilon ) | / \xi  \bigr )   + d  \log (1 + N / \lambda )  \Bigr )  + 8  N \epsilon^2 \\
 & \leq 2 H^2 \Bigl ( 2  \log (H / \xi ) +  4 d^{2}  \log (64 c^2  d^{3 / 2}  N \zeta ) + d  \log (1 + N)\Bigr )  + 8  H^2 d \\ & \leq 2 H^2 \Bigl (2 \zeta  + 4 d^{2}  \log (64 c^2) + 8 d^{2}\zeta +  d \zeta + 4  d \Bigr) \\
 & \leq  d^2  H^2 \zeta \Bigl( 18 (1+1/d) +  8 \log (64 c^2) \Bigr) \leq  d^2  H^2 c^2 \zeta /4
\end{aligned}
$$
when $ c \geq 1$ sufficiently large. We then have $|(\mathrm{ii})| \leq \alpha/2 \sqrt{\phi(x, a)\trans (\Sigma_h^{k})^{-1} \phi(x, a)}$. In addition, $ H \sqrt{d \lambda}  \leq  \alpha/2$, we have 
$$
|\big(\B^k_{h} \widehat{V}^k_{h+1}\big)(x, a)-\big(\widehat{\B}^k_{h} \widehat{V}^k_{h+1}\big)(x, a)|  \leq \alpha  \sqrt{\phi(x, a)\trans (\Sigma_h^{k})^{-1} \phi(x, a)} = \Gamma_{h}^k(x, a),
$$
which finishes the proof. 

\section{Proof of Theorem 3}

At first, we can show that the explicit expression for  $\widehat \theta_{h,k}^0$ and $\widehat \theta_{h,k}^k$ is 
    \begin{equation}
    \begin{aligned}
       \big( (\widehat \theta_{h,k}^0)\trans, (\widehat \theta_{h,k}^k)\trans \big)\trans =  (\Sigma_{h}^k)^{-1} \Big( & \sum\nolimits_{j\neq k}^K   \big( Y_h^j(\widetilde V_{h+1}^j) \big)\trans \big( \bI_{n_j} - \bP_{\Phi_{1,h}^j} \big)\Phi_{0,h}^j +   \big(Y_h^k(\widehat V_{h+1}^k)\big) \trans \Phi_{0,h}^k , \\ & \big(Y_h^k(\widehat V_{h+1}^k)\big) \trans \Phi_{1,h}^k \Big)\trans \,.
   \end{aligned} 
   \label{mini}
    \end{equation}
The minimizer (10) satisfies the following conditions
    \begin{align*}
 &   \frac{\partial f_h}{\partial \theta^j} = -2 (\Phi_{1,h}^j)\trans  \big( Y_h^j(\widehat V_{h+1}^j) - \Phi_{0,h}^j \theta^0 - \Phi_{1,h}^j \theta^j \big ) = \bzero_{d_1}, j \neq k,\\
&     \frac{\partial f_h}{\partial \theta^k} = -2 (\Phi_{1,h}^k)\trans  \big( Y_h^k(\widehat V_{h+1}^k) - \Phi_{0,h}^k \theta^0 - \Phi_{1,h}^k \theta^k \big ) + 2 \lambda \theta^k = \bzero_{d_1} ,\\
& \frac{\partial f_h}{\partial \theta^0} = -2 \sum_{j=1}^K  (\Phi_{0,h}^j)\trans  \big( Y_h^j(\widehat V_{h+1}^j) - \Phi_{0,h}^j \theta^0 - \Phi_{1,h}^j \theta^j \big ) + 2 \lambda \theta^0 = \bzero_{d_0}.
\end{align*}
We then have 
\begin{align*}
   \widehat \theta^j_h & = 
\big( (\Phi_{1,h}^j)\trans   \Phi_{1,h}^j \big)^{+} (\Phi_{1,h}^j)\trans \big( Y_h^j(\widetilde V_{h+1}^j) - \Phi_{0,h}^j \widehat \theta^0_h \big )  \\
  \widehat \theta^k_h & = \big( (\Phi_{1,h}^k)\trans   \Phi_{1,h}^k + \lambda \bI_{d_1} \big)^{-} (\Phi_{1,h}^k)\trans  \big( Y_h^k(\widehat V_{h+1}^k) - \Phi_{0,h}^k \widehat \theta^0_h \big ) \\
 \widehat \theta^0_h & = \big( \sum_{j=1}^K( \Phi_{0,h}^j)\trans   \Phi_{0,h}^j + \lambda \bI_{d_0} \big)^{-}   \Big( \sum_{j\neq k}  (\Phi_{0,h}^j)\trans  \big( Y_h^j(\widetilde V_{h+1}^j) - \Phi_{1,h}^j   \widehat \theta^j_h \big ) + (\Phi_{0,h}^j)\trans  \big( Y_h^k(\widehat V_{h+1}^k) - \Phi_{1,h}^k  \widehat \theta^k_h \big ) \Big),\end{align*}
which further implies that 
\begin{align*}
\big( (\Phi_{1,h}^k)\trans   \Phi_{1,h}^k + \lambda \bI_{d_1} \big)   \widehat \theta^k_h &  = (\Phi_{1,h}^k)\trans Y_h^k(\widehat V_{h+1}^k)   - (\Phi_{1,h}^k)\trans  \Phi_{0,h}^k \widehat \theta^0_h  \\
\big( \sum_{j=1}^K( \Phi_{0,h}^j)\trans   \Phi_{0,h}^j + \lambda \bI_{d_0} \big)  \widehat \theta^0_h =  &   \sum_{j\neq k}  (\Phi_{0,h}^j)\trans  \Big( Y_h^j(\widetilde V_{h+1}^j) - \bP_{\Phi_{1,h}^j} \big( Y_h^j(\widetilde V_{h+1}^j) - \Phi_{0,h}^j \theta^0_h \big ) \Big )     \\ & + (\Phi_{0,h}^k)\trans Y_h^k(\widehat V_{h+1}^k) - (\Phi_{0,h}^k)\trans \Phi_{1,h}^k \widehat \theta^k_h.
\end{align*}
As a result, 
\begin{align*}
    & \Big( \sum_{j \neq k}   ( \Phi_{0,h}^j)\trans ( \bI_{n_j} - \bP_{\Phi_{1,h}^j} )\Phi_{0,h}^j + (\Phi_{0,h}^k)\trans \Phi_{0,h}^k + \lambda \bI_{d_0} \Big)  \widehat \theta^0_h + (\Phi_{0,h}^k)\trans \Phi_{1,h}^k \widehat \theta^k_h \\
    & =  \sum_{j\neq k}  (\Phi_{0,h}^j)\trans (\bI_{n_j} - \bP_{\Phi_{1,h}^j}) Y_h^j(\widetilde V_{h+1}^j) + (\Phi_{0,h}^k)\trans Y_h^k(\widehat V_{h+1}^k) \\
 & (\Phi_{1,h}^k)\trans  \Phi_{0,h}^k \widehat \theta^0_h + \big( (\Phi_{1,h}^k)\trans   \Phi_{1,h}^k + \lambda \bI_{d_1} \big)   \widehat \theta^k_h  = (\Phi_{1,h}^k)\trans Y_h^k(\widehat V_{h+1}^k)  .
\end{align*}
We then get \eqref{mini} by rewriting the two equations into the matrix formula. 
By the proof of Theorem 2, we have 
$$
\begin{aligned}
& \big( (\widehat{\theta}_{h,k}^0)\trans,(\widehat{\theta}_{h,k}^k)\trans \big)\trans - \big( ({\theta}_{h}^0)\trans,({\theta}_{h}^k)\trans \big)\trans  = \lambda (\Sigma_h^{k})^{-1}  \big( (\theta_h^0)\trans, (\theta_h^k)\trans \big)\trans \\
& +  (\Sigma_h^{k})^{-1} \begin{bmatrix}
(\Phi_{0,h}^k)\trans \epsilon_{h}^k - \sum_{j \neq k}^{K}  (\Phi_{0,h}^j)\trans ( \bI_{n_j} - \bP_{\Phi_{1,h}^j} ) \epsilon_{h}^j \\
(\Phi_{1,h}^k)\trans \epsilon_{h}^k 
\end{bmatrix}.     
\end{aligned}
$$
Following the proof of  Theorem 2, we can easily get
$$\bigl \| \big( (\widehat{\theta}_{h,k}^0)\trans,(\widehat{\theta}_{h,k}^k)\trans \big)\trans - \big( ({\theta}_{h}^0)\trans,({\theta}_{h}^k)\trans \big)\trans  \bigr \| \leq \alpha \big\|(\Sigma_h^{k})^{-1}\big\|^{1/2}_{\rm op},$$
which implies that $ \| \widehat{\theta}_{h,k}^0 - \theta_{h}^0\| \leq  \alpha \big\|(\Sigma_h^{k})^{-1}\big\|^{1/2}_{\rm op}$. 

\section{Proof of Corollary 1}

Following from Theorem 2, it holds that
\begin{align}
    \label{eq:co0}
    \operatorname{SubOpt}^k\big(\widehat \pi^k ;x \big) &\leq 2 \alpha \sum_{h=1}^H\E_{k,\pi^{k,*}}\Bigl[ \sqrt{ \phi(x_h,a_h)\trans ( \Sigma_h^k)^{-1}\phi(x_h,a_h) } \mid x_1=x\Bigr] .
\end{align}
Thus, it suffices to upper bound the right-hand side of \eqref{eq:co0}.
By Lemma \ref{lem:concen}, when $n_k$ is sufficiently large, it holds for any $k \in [K], h \in [H]$ that
\begin{align*}
    \Big\| n_k^{-1} \Lambda_h^k - \E_{k} \bigl[ \phi(x_h, a_h) \phi(x_h, a_h)\trans   \bigr] \Big\| \le b/ 2
\end{align*}
with probability at least $ 1 - \xi$. It then holds for any $k \in [K], h \in [H]$ that
\begin{align}\label{eq:lam_min}
     \lambda_{\min}(\Lambda_h^k) \ge
     b n_k / 2
\end{align}
with probability at least $ 1 - \xi$. In what follows, we conditioned on the event defined in \eqref{eq:lam_min}. By the definition of $\Lambda_h^j$ in (7), we notice that the matrix $(\Phi_{0,h}^j)\trans \big( \bI_{n_j} - \bP_{\Phi_{1,h}^j} \big)\Phi_{0,h}^j$ is the Schur complement of $(\Phi_{0,h}^j)\trans \Phi_{0,h}^j$ in $\Lambda_h^j$. Thus, following from \eqref{eq:lam_min}, we have that
\begin{align*}
     \lambda_{\min}\Bigl( (\Phi_{0,h}^j)\trans \big( \bI_{n_j} - \bP_{\Phi_{1,h}^j} \big)\Phi_{0,h}^j\Bigr) \ge n_j   b/ 2.
\end{align*}
Thus, by the definition of $\Sigma_h^k$ in (11), we have that
\begin{align*}
    \Sigma_h^{k}  &= \Lambda_h^{k} + \lambda \bI_d +     \begin{bmatrix}
 \sum_{j \neq k}^{K}  ( \Phi_{0,h}^j)\trans \big( \bI_{n_j} - \bP_{\Phi_{1,h}^j} \big)\Phi_{0,h}^j  & \mathbf{0}_{d_0 \times d_1}\\
\mathbf{0}_{d_1 \times d_0}  & \mathbf{0}_{d_1 \times d_1}
\end{bmatrix} \nonumber \\
& \ge \begin{bmatrix}
 (\lambda + N  b / 2) \bI_{d_0}  & \\
  & ( \lambda + n_k  b / 2) \bI_{d_1}
\end{bmatrix},
\end{align*}
which further implies that 
\begin{align*}
    \label{eq:co1}
    (\Sigma_h^{k})^{-1} 
& \le \begin{bmatrix}
 (\lambda + N  b / 2)^{-1} \bI_{d_0}  & \\
  & ( \lambda + n_k  b / 2)^{-1} \bI_{d_1}
\end{bmatrix},
\end{align*}
which further implies that
\begin{align*}
    \phi(x_h,a_h)\trans ( \Sigma_h^k)^{-1}\phi(x_h,a_h) &\le \big\| \phi_0(x_h, a_h) \big\|^2 / (\lambda + N b / 2) + \big\| \phi_1(x_h, a_h) \big\|^2 / (\lambda + n_k  b / 2) \nonumber \\
    &\le  (\lambda + N  b / 2)^{-1} d_0/d +  (\lambda + n_k b / 2)^{-1} d_1/d \,.
\end{align*}
Combining with \eqref{eq:co0}, we have with probability at least $ 1- 2 \xi$ that
\begin{align*}
    \operatorname{SubOpt}^k\big(\widehat \pi^k ;x \big) &\leq 2 \alpha \sum_{h=1}^H\E_{k,\pi^{k,*}}\Bigl[ \bigl( \phi(x_h,a_h)\trans ( \Sigma_h^k)^{-1}\phi(x_h,a_h) \bigr)^{1/2} \mid x_1=x\Bigr] \nonumber\\
    & \le 2\alpha H \sqrt{ (\lambda + N  b / 2)^{-1} d_0/d +  (\lambda + n_k b / 2)^{-1} d_1/d },
\end{align*}
which proves the first part of Corollary 1. 

We then show the upper bound of $\|\widehat \theta_{h,k}^0 - \theta_h^0 \|$. Define $\bA = (\Phi_{0,h}^k)\trans \Phi_{0,h}^k + \sum_{j \neq k}^{K}  (\Phi_{0,h}^j)\trans \big( \bI_{n_j} - \bP_{\Phi_{1,h}^j} \big) \Phi_{0,h}^j + \lambda \bI_{d_0}$, $\bB = (\Phi_{0,h}^k)\trans \Phi_{1,h}^k$ and $\bD = (\Phi_{1,h}^k)\trans \Phi_{1,h}^k + \lambda \bI_{d_1}$. By the inverse of the partitioned matrix, we then have 
$$(\Sigma_h^{k})^{-1} = \begin{bmatrix}
\bA^{-1} - \bA^{-1} \bB( \bD - \bB\trans \bA^{-1} \bB)^{-1} \bB\trans \bA^{-1} &  -  \bA^{-1} \bB( \bD - \bB\trans \bA^{-1} \bB)^{-1} \\
- (\bD - \bB\trans \bA^{-1} \bB)^{-1} \bB\trans \bA^{-1} &  (\bD - \bB\trans \bA^{-1} \bB)^{-1}
\end{bmatrix}  \,.
$$
Let $\bC = \bD - \bB\trans \bA^{-1} \bB$. We then have 
$$
\lambda_{\min}( \bC ) \geq \lambda_{\min}(\bD) - \|\bB\|^2\|\bA^{-1}\| \geq \frac{b n_k}{2} - \frac{ n_k^2 (1+b/2)^2 }{\lambda+  Nb/2 } \geq \frac{b n_k}{4}
$$
for $K$ sufficiently large. By the proof of Theorem 3, we have 
$$
\begin{aligned}
 \widehat \theta_{h,k}^0 - \theta_h^0 =  & \alpha \Big(  \big( \bA^{-1} - \bA^{-1} \bB \bC^{-1} \bB\trans \bA^{-1} \big) \theta_h^0 - \bA^{-1} \bB \bC^{-1} \theta_h^k \Big) \\
& + \big( \bA^{-1/2} - \bA^{-1} \bB \bC^{-1} \bB\trans \bA^{-1/2} \big) \bA^{-1/2} \big( (\Phi_{0,h}^k)\trans \epsilon_{h}^k - \sum_{j \neq k}^{K}  (\Phi_{0,h}^j)\trans \big( \bI_{n_j} - \bP_{\Phi_{1,h}^j} \big) \epsilon_{h}^j\big) \\
& + \bA^{-1} \bB \bC^{-1} \bD^{1/2} \bD^{-1/2} (\Phi_{1,h}^k)\trans \epsilon_{h}^k \,.
\end{aligned}
$$
By following the proof of Theorem 2, it is not hard to show that 
$$
\begin{aligned}
 \| \widehat \theta_{h,k}^0 - \theta_h^0\| \leq  & \frac{\alpha}{\lambda +  Nb/2} 
 \Big( \big( 1   +   \frac{4  n_k^2(1+b/2)^2 }{b n_k(\lambda +  Nb/2)}   \big)  \sqrt{d_0} + \frac{4  n_k(1+b/2) }{b n_k} \sqrt{d_1} \Big) \\
& + \frac{\alpha}{3 \sqrt{\lambda +  Nb/2}} \Big ( 1+ \frac{4 n_k^2(1+b/2)^2 }{b n_k ( \lambda +  Nb/2 ) } \Big)  + \frac{\alpha}{3 \sqrt{\lambda +  Nb/2}} \frac{4 n_k (1+b/2) \sqrt{\lambda + b n_k /2 } }{ b n_k \sqrt{ \lambda +  Nb/2 } } \Big) \\
& \leq \frac{\alpha}{\sqrt{\lambda + N b /2}} \leq \frac{\sqrt{2} \alpha}{\sqrt{N b}} 
\end{aligned}
$$
when $N$ is sufficiently large. Finally,
$$\|\widehat \theta_{h}^0 - \theta_h^0 \|  \leq
\frac{1}{K} \sum_{k=1}^K \| \widehat \theta_{h,k}^0 - \theta_h^0 \|
\leq \max_{k \in [K]}  \| \widehat \theta_{h,k}^0 - \theta_h^0 \| \leq  \frac{\sqrt{2} \alpha}{\sqrt{N b}}\,,$$
which finishes the proof.

\section{Proof of Corollary 2} 
We have that
\begin{align*}
    \lambda_{\min} \bigl( \sum_{j \neq k} (\Csc_{00, h}^j - \Csc_{01, h}^j (\Csc_{11, h}^j)^{-1} \Csc_{10, h}^j) \bigr) &\ge \lambda_{\min} ( \sum_{j \neq k} \Csc_{00, h}^j) - \sum_{j \neq k}\lambda_{\max}(\Csc_{01, h}^j (\Csc_{11, h}^j)^{-1} \Csc_{10, h}^j) \\
    & \ge (K -1) b / 2  .
\end{align*}
Then, similar to the proof of Corollary 1, we complete the proof of Corollary 2.

\section{Proof of Theorem 4} 
Here we define another basis functions
$$\phi_{2}(x_{1 (h+1)}) = \big( \psi_{21}(x_{1(h+1)}), \ldots, \psi_{2 d_2}(x_{1(h+1)}) \big) \trans \in \R^{d_2}$$
with the number of bases $d_2$ and $\phi_{3}(x_h,a_h,x_{1(h+1)}) = {\rm Vec}\big( \phi_{1}(x_h,a_h) \phi_{2}(x_{1(h+1)})\trans \big) \in \R^{d_3}$, where $d_3 = d_1 d_2$ and the function ${\rm Vec}(\cdot)$ vectorizes a matrix by concatenating its rows. By Theorem 5.2 (iv) of \cite{schultz1969multivariate}, 
there exists  $\theta_h^0 \in \R^{d_0}$, $\theta_h^k \in \R^{d_1}$ and $\bmu_h^k \in \R^{d_1 \times d_2}$ 
such that	
	\begin{equation}
\big| \P^k_{h}(x_{1 (h+1)} \mid x_{h}, a_h) - \big\langle \phi_3(x_h, a_h, x_{1(h+1)}), {\rm Vec}(\bmu_{h}^k) \big\rangle \big| \leq \eta,
\label{S8}
\end{equation}
\begin{equation}
   \big | \E \big[ r^k_{h}(x_{h}, a_{h}) \mid x_{h}, a_{h} \big] - \big\langle \phi_0(x_{h}, a_h), \theta_{h}^0\big\rangle - \big\langle\phi_1(x_h, a_h), \theta_{h}^k\big\rangle \big| \leq \eta \,,
\end{equation}
$\forall ( x_{0 (h+1)}\trans, x_{1 (h+1)}\trans )\trans \in \Ssc, x_{h} \in \Ssc, a_h \in \Asc$, 
where $\eta$ satisfying $\eta \leq C_q m  d^{-(q+1)/m}$ for some constant $C_q$ dependent on $q$. By the boundedness of $\Ssc$, we have  
\begin{equation}
\big\| \P^k_{h}( \cdot \mid x_{h}, a_h) - \langle \phi_3(x_h, a_h, \cdot ), {\rm Vec}(\bmu_{h}^k) \rangle \big\|_{\rm TV} \leq C \eta, \; \forall (x_h,a_h) \in \Ssc \times \Asc 
\label{t2}
\end{equation}
for some constant $C$ dependent on the measure of $\Ssc$, where $\| \cdot \|_{\rm TV}$ is the total variation distance. Since we only consider the rate of $\eta$, for simplicity, we assume $C = 1$ (otherwise, $C$ can be absorbed into $C_q$). 

By Lemma 3.1 of \cite{jin2020pessimism}, we have 
\begin{align*}
\operatorname{SubOpt}^k\big(\widehat \pi^k ;x \big) = & \underbrace{-\sum_{h=1}^H \mathbb{E}_{k, \widehat{\pi}^k }\left[\iota_h^k (x_h, a_h) \mid x_1= x\right]}_{\text {(A): Spurious Correlation }}+ \underbrace{ \sum_{h=1}^H \mathbb{E}_{k,\pi^{k,*}}\left[\iota_h^k(x_h, a_h) \mid x_1=x\right]}_{( \text {B}): \text { Intrinsic Uncertainty }} \\
& +\underbrace{\sum_{h=1}^H \mathbb{E}_{k,\pi^{k,*}}\left[ \langle \widehat{Q}_h^k(x_h, \cdot), \pi_h^{k,*}(\cdot \mid x_h) - \widehat{\pi}_h^k (\cdot \mid x_h) \rangle_{\mathcal{A}} \mid x_1=x \right]}_{\text {(C): Optimization Error }}\,,
\end{align*}
where $\iota_h^k (x, a) = (\B_h^k \widehat V_{h+1}^k)(x,a) - \widehat{Q}_h^k(x, a)$. By the definition of $\widehat \pi_h^k$, we have $\text{(C)} \leq 0$. We will show 
\begin{equation}
\begin{aligned}
 0  \leq \iota_h^k (x, a) \leq 2 \overline \Gamma_h^k(x,a) 
 \label{iota}
\end{aligned}
\end{equation}
with probability at least $1 - \xi$, 
which yields 
\begin{align*}
\operatorname{SubOpt}^k\big(\widehat \pi^k ;x \big) & \leq  2 \sum_{h=1}^H \mathbb{E}_{k,\pi^{k,*}}\big[ \overline \Gamma_h^k(x_h,a_h) \mid x_1=x\big] 	\,.
\end{align*}
Before proving \eqref{iota}, we bound $\big|\big(\B_{h}^k \widehat{V}^k_{h+1} \big)(x, a)- \big(\widehat \B_{h}^k \widehat{V}^k_{h+1} \big)(x, a)\big|$ first. By definition, for $k \in [K]$,  we have 
\begin{equation*}
\begin{aligned}
\big(\B_{h}^k \widehat{V}^k_{h+1} \big)(x, a) =& \E_k\big[r^k_{h}\big(x_{h}, a_{h}\big) + \widehat{V}^k_{h+1}(x_{h+1}) \mid x_{h}=x, a_{h}=a\big] \\
= & \E_k \big[r_{h}^k(x_{h}, a_{h}) \mid x_{h}=x, a_{h}=a\big] + \int_{x^{\prime} \in \Ssc} \widehat{V}^k_{h+1}(x^{\prime})  \P^k_{h}(x^{\prime} \mid x_h = x, a_h = a) {\rm d} x^{\prime}  \\
= &  \langle\phi_0(x, a), \theta_{h}^0 \rangle + \langle\phi_1(x, a), \bar \theta_{h}^k \rangle + e_h^k(x,a) \,,
\end{aligned}   
\end{equation*}
where $\bar \theta_h^k = \theta_h^k +  \int_{ x_{1(h+1)} }  \bmu_h^k  \phi_{2}(x_{1(h+1)}) \widehat{V}^k_{h+1}( (x_{0(h)}\trans ,x_{1(h+1)}\trans )\trans ) \mathrm{d} x_{1(h+1)}$ and 
$e_h^k(x,a)$ satisfies
\begin{equation}
\label{bound e}
\begin{aligned}
|e_h^k(x,a)| = &  \big| (\B_{h}^k \widehat{V}^k_{h+1} )(x, a)  - \langle\phi_0(x, a), \theta_{h}^0 \rangle - \langle\phi_1(x, a),  \bar \theta_{h}^k \rangle \big| \\
\leq &  \big|  \E_k \big[r_{h}^k(x_{h}, a_{h}) \mid x_{h}=x, a_{h}=a\big] - \langle\phi_0(x, a), \theta_{h}^0 \rangle - \langle\phi_1(x, a),  \theta_{h}^k \rangle \big|  \\
& + \Big| \int_{x^{\prime} \in \Ssc} \widehat{V}^k_{h+1}(x^{\prime}) \P^k_{h}(x^{\prime} \mid x_h=x, a_h=a) {\rm d} x^{\prime}  -   \langle\phi_1(x, a),  \bar \theta_{h}^k - \theta_{h}^k \rangle \big| \Big| \\
\leq &  \eta + \max_{x^{\prime} \in \Ssc} |\widehat{V}^k_{h+1}(x^{\prime})| \big\| \P^k_{h}( \cdot \mid x_h=x, a_h=a) - \langle \phi_3(x_h, a_h, \cdot ), {\rm Vec}(\bmu_{h}^k) \rangle  \big\|_{\rm TV} \\
\leq &  (H+1) \eta \,,
\end{aligned}
\end{equation}
by \eqref{S8} and \eqref{t2}, and the fact that $|\widehat{V}^k_{h+1}(x)| \leq H$. As a result, we know that there exists $\theta_h = \big( (\theta_h^0)\trans, (\bar \theta_h^1)\trans,\ldots, (\bar \theta_h^K)\trans \big)\trans \in \R^{d_K}$ such that $\big( \B^k_{h} \widehat{V}^k_{h+1} \big) (x,a) = \phi^k(x, a) \trans  \theta_h + e_h^k(x,a)$ and $\big( \B^j_{h} \widetilde{V}^j_{h+1} \big) = \phi^j(x, a) \trans  \theta_h + e_h^j (x,a)$, for $j \in [K]/\{k\}$. Now we  bound the difference between $\B^k_{h} \widehat{V}^k_{h+1}$ and $\widehat{\B}^k_{h} \widehat{V}^k_{h+1}$. Define $Y_h \in \R^N$ the same as Section S2. For all $h \in[H]$ and all $(x, a) \in \mathcal{S} \times \Asc$, we have
$$
\begin{aligned}
& (\B^k_{h} \widehat{V}^k_{h+1})(x, a)-\big(\widehat{\B}^k_{h} \widehat{V}^k_{h+1})(x, a) = \phi^k(x, a) \trans ( \theta_h - \widehat \theta_{h,k}) + e_h^k(x,a)\\
= &  \phi^k(x, a) \trans  \theta_h  -  \phi^k(x, a) \trans(\Lambda_h  +\bH_{k,\lambda})^{+} \Phi_h\trans  Y_h + e_h^k(x,a)\\
= & \underbrace{\phi^k(x, a)\trans  \theta_h - \phi^k(x, a)\trans  (\Lambda_h  +\bH_{k,\lambda})^{+} \Phi_h\trans  \Phi_h   \theta_h}_{(\mathrm{iv})}\\
 & -\underbrace{\phi^k(x, a)\trans (\Lambda_h  +\bH_{k,\lambda})^{+}  \Phi_h\trans ( Y_h - \E Y_h )}_{(\mathrm{v})}  -\underbrace{\phi^k(x, a)\trans (\Lambda_h  + \bH_{k,\lambda})^{+}  \Phi_h\trans (   \E Y_h - \Phi_h \theta_h)}_{(\mathrm{vi})} + e_h^k(x,a).
\end{aligned}
$$
Thus, by triangle inequality, we have
$$
\big|(\B_{h}^k \widehat{V}^k_{h+1})(x, a)-(\widehat{\B}_{h}^k \widehat{V}^k_{h+1})(x, a)\big| \leq |(\mathrm{iv})|+|(\mathrm{v})| +|(\mathrm{vi})| + |e_h^k(x,a)|.
$$	
Similar to the proof of Theorem 2, we can show $ |(\mathrm{iv})| \leq H \sqrt{ d / \lambda}  \sqrt{\phi(x, a)\trans(\Sigma_h^{k})^{-1} \phi(x, a)} \leq \alpha/2  \sqrt{\phi(x, a)\trans(\Sigma_h^{k})^{-1} \phi(x, a)}$ and $|(\mathrm{v})| \leq \alpha/2  \sqrt{\phi(x, a)\trans(\Sigma_h^{k})^{-1} \phi(x, a)}$ with probability at least $1 - \xi$. Denote $u_h^k = \E Y_h^k(\widehat V_{h+1} ^k) - \Phi_h^k \theta_h$ and  $u_h^j = \E Y_h^j(\widetilde V_{h+1} ^j) - \Phi_h^j \theta_h$ for $j \in [K] / \{k\}$. 
We then bound $|(\mathrm{vi})|$ as follows: 
$$
\begin{aligned}
& \phi^k(x, a)\trans (\Lambda_h  + \bH_{k,\lambda})^{+}  \Phi_h\trans (   \E Y_h - \Phi_h \theta_h) \\
&   = \left | \phi(x, a)\trans  (\Sigma_h^{k})^{-1} \begin{bmatrix}
(\Phi_{0,h}^k)\trans u_h^k  - \sum_{j \neq k}^{K}  (\Phi_{0,h}^j)\trans ( \bI_{n_j} - \bP_{\Phi_{1,h}^j} ) u_h^j\\
(\Phi_{1,h}^k)\trans u_h^k
\end{bmatrix} \right | \\
& \leq \sqrt{\sum_{j=1}^K\|u_h^j\|_2^2
} \sqrt{ \phi(x, a)\trans  (\Sigma_h^{k})^{-1}  \phi(x, a)} \leq 
\sqrt{N  } (H +1) \eta \sqrt{ \phi(x, a)\trans  (\Sigma_h^{k})^{-1}  \phi(x, a)}  \end{aligned} \,.
$$
Combining with \eqref{bound e}, we have
\begin{equation}
\begin{aligned}
& \big|(\B_{h}^k \widehat{V}^k_{h+1})(x, a)-(\widehat{\B}_{h}^k \widehat{V}^k_{h+1})(x, a)\big| \\
& \leq \big( \alpha + \sqrt{N}(H+1) \eta\big)  \sqrt{\phi(x, a)\trans  (\Sigma_h^{k})^{-1}  \phi(x, a)}	+ (H +1) \eta = \overline \Gamma_{h}^k(x, a) .
\end{aligned}
\label{eq:diff}
\end{equation}
We now prove \eqref{iota}. Recall the construction of $\widehat Q_h^k$ in Line $6$ of Algorithm 2. For all $h \in [H]$ and all $(x, a) \in \Ssc \times \Asc$, if $(\widehat{\B}_{h}^k \widehat{V}^k_{h+1})(x, a) - \overline \Gamma_{h}^k(x, a) < 0$, we have 
$$\widehat{Q}^k_{h}(x, a) = \min \big\{ (\widehat{\B}_{h}^k \widehat{V}^k_{h+1})(x, a) - \overline \Gamma_{h}^k(x, a), H-h+1\big\}^{+} = 0$$
and $\iota_h^k (x, a) = (\B_h^k \widehat V_{h+1}^k)(x,a) - \widehat{Q}_h^k(x, a) =  (\B_h^k \widehat V_{h+1}^k)(x,a)  \geq 0$
due to the non-negativity of $r_h^k$ and $\widehat V_h^k$. Otherwise, we have $\widehat{Q}^k_{h}(x, a) \leq (\widehat{\B}_{h}^k \widehat{V}^k_{h+1})(x, a)- \widehat \Gamma_{h}^k(x, a)$, which yields
\begin{align*}
\iota_h^k (x, a) = & (\B_h^k \widehat V_{h+1}^k)(x,a) - \widehat{Q}_h^k(x, a) \geq (\B_h^k \widehat V_{h+1}^k)(x,a) - (\widehat{\B}_{h}^k \widehat{V}^k_{h+1})(x, a)+ \widehat \Gamma_{h}^k(x, a) \\
= & (\B^k_{h} \widehat{V}^k_{h+1})(x, a)-\big(\widehat{\B}^k_{h} \widehat{V}^k_{h+1})(x, a) +  \overline \Gamma_{h}^k(x, a) \geq 0 \,,\\
\end{align*}
where the last inequality comes from \eqref{eq:diff}. We then show the upper bound of $\iota_h^k (x, a)$. By \eqref{eq:diff} and the fact $r_h^k \in [0,1]$ and $\widehat V_h^k \in [0,H-h]$, we have 
\begin{align*}
(\widehat{\B}_{h}^k \widehat{V}^k_{h+1})(x, a)- \overline \Gamma_{h}^k(x, a) &  \leq \big({\B}^k_{h} \widehat{V}^k_{h+1})(x, a)  \leq H-h+1 \,,
\end{align*}
which implies 
\begin{align*}
\widehat{Q}^k_{h}(x, a) & = \min \big\{ (\widehat{\B}_{h}^k \widehat{V}^k_{h+1})(x, a) - \widehat \Gamma_{h}^k(x, a), H-h+1\big\}^{+}  \geq (\widehat{\B}_{h}^k \widehat{V}^k_{h+1})(x, a) - \overline \Gamma_{h}^k(x, a) 
\end{align*}
and 
\begin{align*}
\iota_h^k (x, a) & = (\B_h^k \widehat V_{h+1}^k)(x,a) - \widehat{Q}_h^k(x, a)   \leq (\B_h^k \widehat V_{h+1}^k)(x,a) -  (\widehat{\B}_{h}^k \widehat{V}^k_{h+1})(x, a) + \overline \Gamma_{h}^k(x, a)   \leq 2 \overline \Gamma_{h}^k(x, a) \,,
\end{align*}
which arrives the upper bound of $\iota_h^k (x, a)$. 
Finally, we need to decide the rate of $J$ to finish the proof. Recall that 
$\alpha = c d H \sqrt{\zeta}$ where $\zeta= \log(2dHN/\xi)$. By solving $d \approx \sqrt{N}  \eta \approx \sqrt{N} m  d^{-(q+1)/m}$, we have the rate of $d$ is $N^{ \frac{m}{2(m+q+1)} } m^{\frac{m}{m+q+1}}$. Thus, by setting $d = c N^{ \frac{m}{2(m+q+1)} } m^{\frac{m}{m+q+1}}$ for some constant $c$, we have 
\begin{align*}
\operatorname{SubOpt}^k\big(\widehat \pi^k ;x \big)  \leq & 2 \big( \alpha + \sqrt{N}(H+1) \eta\big)     \sum_{h=1}^H \mathbb{E}_{k,\pi^{k,*}}\big[ \sqrt{\phi(x, a)\trans  (\Sigma_h^{k})^{-1}  \phi(x, a)} \mid x_1=x\big] \\
& + 2 (H +1) \eta \\
 \leq & C H N^{ \frac{m}{2(m+q+1)} } m^{\frac{m}{m+q+1}} \sqrt{\zeta}   \sum_{h=1}^H \mathbb{E}_{k,\pi^{k,*}}\big[ \sqrt{\phi(x, a)\trans  (\Sigma_h^{k})^{-1}  \phi(x, a)} \mid x_1=x\big] \\
 & + C H N^{ \frac{-(q+1)}{2(m+q+1)}  } m^{\frac{m}{m+q+1}} \,.
\end{align*}

\section{Proof of Corollary 3}

Its proof follows straightforwardly from the proof of Corollary 1 and Theorem 4, thus omitted here. 

\section{Model Misspecification}
\label{supp:mod mis}
In reality, algorithms designed for a linear MDP could perform poorly if the linear assumption is not satisfied. However, 
for FDTR, the following Theorem \ref{theorem: mis} shows its robustness to model misspecification. Instead of assuming (1) and (2), we allow the underlying MDP to deviate from the linear MDP structure by assuming that 
\begin{equation}
\big\| \P^k_{h}(x_{h+1} \mid x_h, a_h)- \langle \phi_1(x_h, a_h), \mu_{h}^k (x_{h+1})\rangle  \big\|_{\rm TV} \leq \Delta,
\label{misp}
\end{equation}
\begin{equation}
   \big |   \E \big[ r^k_{h}(x_{h}, a_{h}) \mid x_{h}, a_{h} \big]- \langle\phi_0(x_h, a_h), \theta_{h}^0 \rangle - \langle\phi_1(x_h, a_h), \theta_{h}^k \rangle
 \big| \leq \Delta \,,
 \label{misr}
\end{equation}
for some $\Delta \in [0,1)$ and $\forall x_h \in \Ssc$ and $\forall a_h \in \Asc$. Here $\Delta$ characterizes the discrepancy of the underlying MDP from a linear MDP. 
The range of $\Delta \in [0,1)$ comes from the fact that $\|\P^k_{h}( \cdot \mid x_{h}, a_h)\|_{\rm TV} =1$ and $|r^k_{h}(x_{h}, a_{h})|\leq 1$. The assumptions \eqref{misp} and \eqref{misr} state that we can use a linear MDP to approximate the underlying MDP, which could be non-linear. When $\Delta = 0$, Theorem \ref{theorem: mis} becomes the same as Theorem 2. 

\begin{theorem}[Suboptimality of FDTR under Misspecified Linear MDP]
\label{theorem: mis}
Assume the transition kernel $\P^k_{h}(\cdot \mid x_h, a_h)$ and the reward function $r^k_{h}(\cdot,\cdot)$ satisfy \eqref{misp} and \eqref{misr} for  some $\Delta \in [0,1)$. In Algorithm 2, we replace $\widehat \Gamma_{h}^k(x, a)$ by $\overline \Gamma_{h}^k(x, a)$ with the tuning parameters $\alpha$ set as (13) and $\eta = \Delta$. Then we have, for any $k \in [K]$ and for any $x \in \Ssc$, $\widehat \pi^k = \{ \widehat \pi_h^k \}_{h=1}^H$ in  Algorithm 2 satisfies    
\begin{equation*}
\begin{aligned}
\operatorname{SubOpt}^k\big(\widehat \pi^k ;x \big) \leq &  2 \big( \alpha + \sqrt{N}(H+1) \Delta\big)  \sum_{h=1}^H\E_{k,\pi^{k,*}}\Bigl[\sqrt{\phi(x_h,a_h)\trans ( \Sigma_h^k)^{-1}\phi(x_h,a_h) }\mid x_1=x\Bigr] \\
& +  2 (H +1) \Delta
\end{aligned} 
\end{equation*}
with probability at least $1- \xi$. 
\end{theorem}

\begin{proof}
    The proof of Theorem \ref{theorem: mis} is similar to the proof of Theorem 4. Specifically, we can get 
\begin{align*}
\operatorname{SubOpt}^k\big(\widehat \pi^k ;x \big)  \leq & 2 \sum_{h=1}^H \mathbb{E}_{k,\pi^{k,*}}\big[ \overline \Gamma_h^k(x_h,a_h) \mid x_1=x\big] 	\\
\leq & 2 \big( \alpha + \sqrt{N}(H+1) \Delta\big)     \sum_{h=1}^H \mathbb{E}_{k,\pi^{k,*}}\big[ \sqrt{\phi(x, a)\trans  (\Sigma_h^{k})^{-1}  \phi(x, a)} \mid x_1=x\big] \\
& + 2 (H +1) \Delta 
\end{align*}
with probability at least $1 - \xi$ 
\end{proof}

\section{Technical Lemmas}

\begin{lemma}[Bounded Coefficients] Let $V_{\max }>0$ be an absolute constant. For any functions $V^j: \mathcal{S} \rightarrow\left[0, V_{\max }\right], \forall j \in [K]$ and any $h \in[H]$, let $\theta_h$ be the weight vector in $R^{d_K}$ satisfying $\B^j_{h} V^j_{h+1} = \phi^j(x, a) \trans  \theta_h, \forall j \in [K]$, then we have
$$
\| \theta_h^j \| \leq (1+V_{\max})\sqrt{d_1}; \quad \bigl \|\big( (\widehat{\theta}_{h,k}^0)\trans,(\widehat{\theta}_{h,k}^k)\trans \big)\trans  \bigr \| \leq H \sqrt{N / \lambda}; \quad \|\widetilde \theta^j_{h}\| \leq H \sqrt{n_k / \lambda}.
$$
where $\widehat \theta_h$ is defined in (10) and $\widetilde \theta^j_{h}$ is defined in Algorithm 1. 

\label{lemma: norm}
\end{lemma}
\begin{proof}
By (5), we have $\theta_h = \big( (\theta_h^0)\trans, (\bar \theta_h^1)\trans,\ldots, (\bar \theta_h^K)\trans \big)\trans$, where 
$$\theta_h^0 = \theta_0; \; \bar \theta_h^j =  \theta_h^j +  \int_{x^{\prime} \in \mathcal{S}} \mu_{h}^j\left(x^{\prime}\right)  V^j \left( x^{\prime} \right) \mathrm{d} x^{\prime}, \forall j \in [K].$$ 
Then we have 
$$\| \bar \theta_h^j \| \leq
 \| \theta_h^j \| +  \big \|\int_{x^{\prime} \in \mathcal{S}} \mu_{h}^j\left(x^{\prime}\right)  V^j \left( x^{\prime} \right) \mathrm{d} x^{\prime} \big\| \leq \sqrt{d_1} + V_{\max} \|\mu_{h}^j(\mathcal{S})\| \leq (1+V_{\max})\sqrt{d_1}.$$
Besides, note that $|r_h^{k,\tau} + \widehat V_{h+1}^{k}\big( x^{k,\tau}_{h+1} \big)| \leq H$ and $|r_h^{j,\tau} + \widetilde V_{h+1}^{j}\big( x^{j,\tau}_{h+1} \big)| \leq H$, we have $\|Y_h\|_{\infty} \leq H$ and hence $\|Y_h\| \leq \sqrt{N} \|Y_h\|_{\infty} \leq \sqrt{N}H$. In addition, 

$$
\big( (\widehat{\theta}_{h,k}^0)\trans,(\widehat{\theta}_{h,k}^k)\trans \big)\trans =  (\Sigma_h^{k})^{-1} \begin{bmatrix}
(\Phi_{0,h}^k)\trans Y_{h}^k - \sum_{j \neq k}^{K}  (\Phi_{0,h}^j)\trans ( \bI_{n_j} - P_{\Phi_{1,h}^j} ) Y_{h}^j \\
(\Phi_{1,h}^k)\trans Y_{h}^k 
\end{bmatrix}.
$$
As a result,
$$
\begin{aligned}
&  \bigl \|\big( (\widehat{\theta}_{h,k}^0)\trans,(\widehat{\theta}_{h,k}^k)\trans \big)\trans  \bigr \|^2 \\
 & \leq \|Y_h\|^2  \bigl \| (\Sigma_h^{k})^{-1}  \begin{bmatrix}
(\Phi_{0,h}^k)\trans \Phi_{0,h}^k + \sum_{j \neq k}^{K}  (\Phi_{0,h}^j)\trans \Bigl( \bI_{n_j} - \bP_{\Phi_{1,h}^j} \Bigr) \Phi_{0,h}^j & (\Phi_{0,h}^k)\trans \Phi_{1,h}^k \\
(\Phi_{1,h}^k)\trans \Phi_{0,h}^k & (\Phi_{1,h}^k)\trans \Phi_{1,h}^k
\end{bmatrix}  (\Sigma_h^{k})^{-1} \bigr \|_{\rm op} \\
& \leq  \|Y_h\|^2 \bigl \|  (\Sigma_h^{k})^{-1} \|_{\rm op}  \leq NH^2/\lambda.
\end{aligned}
$$
Similarly we can prove that $\|\widetilde \theta^j_{h}\| \leq H \sqrt{n_k / \lambda}$. 
\end{proof}

\begin{lemma}[Concentration of Self-Normalized Processes]  \label{lemma: Concentration}
Let $V^j: \mathcal{S} \rightarrow[0, H-1], j \in [K]$ be any fixed functions, $\forall j \in [K]$. For any fixed $h \in[H]$ and any $\delta \in(0,1)$, we have
$$
\Pbb_{\Dsc}\left( f\big( \{ V^{j}\}_{j \in [K]}\big) > H^{2} (2 \log (1 / \delta)+d \log (1+ N/ \lambda))\right) \leq \delta,
$$
where $f\big( \{ V^{j}\}_{j \in [K]}\big)$ is defined in \eqref{def: f}. 
\end{lemma}
\begin{proof}
For the fixed $h \in[H]$ and all $j \in [K], \tau \in [n_j]$, we define the $\sigma$-algebra
$$
\mathcal{F}_{h,j, \tau}=\sigma\big(\big\{\big(x_{h}^{s,i}, a_{h}^{s,i}\big)\big\}_{s,i=1}^{K, n_s} \cup\big\{\big(r_{h}^{s,i}, x_{h+1}^{s,i}\big)\big\}_{s,i=1}^{j,\tau}\big),
$$
where $\sigma(\cdot)$ denotes the $\sigma$-algebra generated by a set of random variables. 
For all $j \in [K], \tau \in[n_j]$, we have $\phi^j(x_{h}^{j,\tau}, a_{h}^{j,\tau}) \in \mathcal{F}_{h, j, \tau-1}$, as
$(x_{h}^{j,\tau}, a_{h}^{j,\tau})$ is $\mathcal{F}_{h,j,\tau-1}$-measurable.  Also,
for the fixed function $V: \mathcal{S} \rightarrow[0, H-1]$ and all $\tau \in[n_j]$, we have 
$$
\epsilon_{h}^{j,\tau}(V)=r_{h}^{j,\tau}+V\big(x_{h+1}^{j,\tau}\big)-\big(\B^j_{h} V\big)\big(x_{h}^{j,\tau}, a_{h}^{j,\tau}\big) \in \mathcal{F}_{h, j, \tau}
$$
as $(r_{h}^{j,\tau}, x_{h+1}^{j,\tau})$ is $\mathcal{F}_{h,j, \tau}$-measurable. Hence, $\big\{\epsilon_{h}^{j,\tau}(V^j)\big\}_{j,\tau=1}^{K,n_j}$ is a stochastic process adapted to the filtration $\left\{\mathcal{F}_{h,j, \tau}\right\}_{j,\tau=0}^{K,n_j}$. We have
$$
\begin{aligned}
 & \E_{\Dsc}[\epsilon_{h}^{j,\tau}(V^j) \mid \mathcal{F}_{h,j, \tau-1}] \\
& =\E_{\Dsc}\left[r_{h}^{j,\tau}+V^j(x_{h+1}^{j,\tau}) \mid \{(x_{h}^{s,i}, a_{h}^{s,i})\}_{s,i=1}^{K,n_s},\{(r_{h}^{s,i}, x_{h+1}^{s,i} \}_{s,i=1}^{j,\tau-1}\right]-(\B^j_{h} V^j) (x_{h}^{j,\tau}, a_{h}^{j,\tau}) \\
&=\E_j \left[r^j_{h}\left(x_{h}, a_{h}\right)+V^j\left(x_{h+1}\right) \mid x_{h}=x_{h}^{j,\tau}, a_{h}=a_{h}^{j,\tau}\right]-(\B^j_{h} V^j) (x_{h}^{j,\tau}, a_{h}^{j,\tau})=0.
\end{aligned}
$$
Here $\E_{\Dsc}$ is taken with respect to $\Pbb_{\Dsc}$, while $\E_j$ is
taken with respect to the immediate reward and next state in the underlying MDP. As a result, $\epsilon_h^{j,\tau}(V^j)$ is mean-zero and $H$-sub-Gaussian conditioning on 
$\mathcal{F}_{h, j, \tau-1}$. 

We invoke Lemma \ref{lem:concen_self_normalized} with $M_{0}=\lambda \bI$ and $M_{N}=\Sigma_h^k$. For the fixed function $V: \mathcal{S} \rightarrow[0, H-1]$ and fixed $h \in[H]$, we have
$$
\mathbb{P}_{\Dsc}\bigl(  f\big( \{ V^{j}\}_{j \in [K]}\big) > 2 H^{2} \cdot \log \big(\frac{\operatorname{det}\big(\Sigma_{h}^k\big)^{1 / 2}}{\delta \cdot \operatorname{det}(\lambda \cdot I)^{1 / 2}}\big) \bigr) \leq \delta
$$
for all $\delta \in(0,1)$.Note that $\|\phi^j(x, a)\| \leq 1$ for all $(x, a) \in \mathcal{S} \times \Asc$.
We have
$\| \Sigma_h^k \|_{\rm op} \leq \lambda + N$.
Hence, it holds that $\operatorname{det}\big (\Sigma_{h}^k \big) \leq(\lambda+ N)^{d}$ an $\operatorname{det}(\lambda \cdot I)=\lambda^{d}$, which implies
$$
\begin{array}{l}
\mathbb{P}_{\Dsc}\bigl(  f\big( \{ V^{j}\}_{j \in [K]}\big) > H^{2} \cdot(2 \cdot \log (1 / \delta)+d \cdot \log (1+ N/ \lambda))\bigr) \\
\quad \leq \mathbb{P}_{\Dsc}\bigl(  f\big( \{ V^{j}\}_{j \in [K]}\big)  > 2 H^{2} \cdot \log \big(\frac{\operatorname{det}(\Lambda_{h})^{1 / 2}}{\delta \cdot \operatorname{det}(\lambda \cdot I)^{1 / 2}}\big)\bigr) \leq \delta
\end{array}
$$
Therefore, we conclude the proof.
\end{proof}

\begin{lemma}[Concentration of Self-Normalized Processes \citep{NIPS2011_e1d5be1c}]
Let $\{\Fsc_t \}^\infty_{t=0}$ be a filtration and $\{\epsilon_t\}^\infty_{t=1}$ be an $\R$-valued stochastic process such that $\epsilon_t$ is $\Fsc_{t} $-measurable for all $t\geq 1$.
Moreover, suppose that conditioning on $\Fsc_{t-1}$, 
 $\epsilon_t $ is a  zero-mean and $\sigma$-sub-Gaussian random variable for all $t\geq 1$, that is,  
 $$
  \E[\epsilon_t \mid \Fsc_{t-1}]=0,\qquad \E\bigl[ \exp(\lambda \epsilon_t) \mid \Fsc_{t-1}\bigr]\leq \exp(\lambda^2\sigma^2/2) , \qquad \forall \lambda \in \R. 
$$
 Meanwhile, let $\{\phi_t\}_{t=1}^\infty$ be an $\R^d$-valued stochastic process such that  $\phi_t $  is $\Fsc_{t -1}$-measurable for all $ t\geq 1$. 
Also, let  $\bM_0 \in \R^{d\times d}$ be a  deterministic positive-definite matrix and 
$$
\bM_t = \bM_0 + \sum_{s=1}^t \phi_s\phi_s\trans
$$ for all $t\geq 1$. For all $\delta>0$, it holds that
\begin{equation*}
\Big\| \sum_{s=1}^t \phi_s \epsilon_s \Big\|_{ \bM_t ^{-1}}^2 \leq 2\sigma^2\cdot  \log \Bigl( \frac{\det(\bM_t)^{1/2} \det(\bM_0)^{- 1/2}}{\delta} \Bigr)
\end{equation*}
for all $t\ge1$ with probability at least $1-\delta$.
\label{lem:concen_self_normalized}
\end{lemma}
\begin{proof}
See Theorem 1 of \cite{NIPS2011_e1d5be1c} for a detailed proof. 
\end{proof}

\begin{lemma}[$\epsilon$-Covering Number \citep{jin2020provably}]
\label{lemma: covering number}
For all $h \in[H]$ and all $\epsilon>0$, we have
$$
\log \left|\mathcal{N}_{h}(\epsilon ; R, B, \lambda)\right| \leq d  \log (1+4 R / \epsilon)+d^{2} \log \left(1+8 d^{1 / 2} B^{2} /\left(\epsilon^{2} \lambda\right)\right)
$$
\end{lemma}
\begin{proof}
See Lemma D.6 in \cite{jin2020provably} for a detailed proof.
\end{proof}

\begin{lemma}[Concentration] \label{lem:concen}
When $n_k$ is sufficiently large, for any $h \in [H]$, it holds for any $h \in [H]$ and $k \in [K]$ that 
\begin{align*}
   \big\| n_k^{-1} \Lambda_h^k - \Csc_h^k \big\| \le C \sqrt{ \log(dHK/\xi) / n_k}
\end{align*}
with probability at least $ 1- \xi$, where $C > 0$ is an absolute constant and the expectation $\E_{\Dsc_k}$ is with respect to the data collecting process at the $k$th site.
\end{lemma}

\begin{proof}
We notice that
\begin{align*}
\frac{\Lambda_h^k}{n_k}
 - \Csc_h^k  = \frac{1}{n_k} \sum_{\tau = 1}^{n_k} \Bigl(\phi(x_h^{k, \tau}, a_h^{k, \tau}) \phi(x_h^{k, \tau}, a_h^{k, \tau})^\top - \E_{\Dsc_k}\bigl[ \phi(s_h, a_h) \phi(s_h, a_h)^\top \bigr] \Bigr) = \frac{\sum_{\tau =  1}^{n_k} Z^{k,\tau}_h}{n_k} \,,
\end{align*}
where we write 
\begin{align*}
    Z^{k,\tau}_h = \phi(x_h^{k, \tau}, a_h^{k, \tau}) \phi(x_h^{k, \tau}, a_h^{k, \tau})^\top - \E_{\Dsc_k} \bigl[ \phi(x_h, a_h) \phi(x_h, a_h)^\top \mid x_1 = x  \bigr].
\end{align*}
We notice that $\|Z_h^{k, \tau}\| \le 2$.
Then, applying the matrix Bernstein inequality \citep{tropp2015introduction}, we have that
\begin{align*}
    \|n_k^{-1} \sum_{\tau = 1}^{n_k} Z_h^{k, \tau}\| \le C \sqrt{\log(d/\xi) / n_k}
\end{align*}
with probability at least $1- \xi$ for an absolute constant $C > 0$. Applying the union bound, by the definition of $Z_h^{k, \tau}$, we have for any $h \in [H]$ and $k \in [K]$ that
\begin{align*}
     \Big\| n_k^{-1} \Lambda_h^k - \E_{\Dsc_k} \bigl[ \phi(x_h, a_h) \phi(x_h, a_h)^\top  \bigr] \Big\| \le C \sqrt{ \log(dHK/\xi) / n_k}
\end{align*}
with probability at least $1-\xi$. Thus, we complete the proof of Lemma \ref{lem:concen}.
\end{proof}

\section{Additional Numeric Analysis}

 We have further conducted sensitivity analysis by using a more complex model. Specifically, for the heterogeneous covariates, we use a second-order polynomial  as 
$$\phi_l(x_h,a_h)= \big( x_{lh},a_h x_{lh},a_h^2 x_{lh}, x_{lh}\circ x_{lh} ,a_h (x_{lh}\circ x_{lh}),a_h^2 (x_{lh} \circ x_{lh}) \big)\,,$$
where $\circ$ denotes the elementwise multiplication of vectors, and $x_{lh}$ consists of the five patient-level covariates: measurements of weight, temperature, systolic blood pressure, hemoglobin, and potassium level. We then get $d_1 = 30$, plus $d_0 = 6$, yielding a state-space dimension of $36$. The estimated value function is presented \ref{as}. 
We can see that including non-linear effects do not improve the value function, in part due to the bias-variance trade-off with a limited sample size. 
\begin{figure}[ht]
    \centering
    \includegraphics[width=0.7\textwidth]{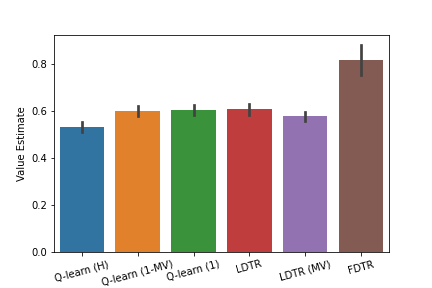}
        \caption{Estimated value function.}
          \label{as}
\end{figure}

\bibliographystyle{chicago}
\bibliography{references}       